\documentclass{article} 
\PassOptionsToPackage{numbers, compress}{natbib}
\usepackage[preprint]{neurips_2020}

\usepackage{amsmath}
\usepackage{amsfonts}

\usepackage{booktabs} 
\usepackage{multirow}

\usepackage{pifont}

\usepackage{enumerate}

\usepackage{graphicx}
\usepackage{tikz}

\usepackage{subcaption}

\newlength\figureheight
\newlength\figurewidth

\makeatletter
\newcommand*\bigcdot{\mathpalette\bigcdot@{.5}}
\newcommand*\bigcdot@[2]{\mathbin{\vcenter{\hbox{\scalebox{#2}{$\m@th#1\bullet$}}}}}
\makeatother

\newcommand{\equationFullStop}{.}

\newcommand{\widthSymbol}{i}
\newcommand{\widthSymbolB}{j}

\newcommand{\inputSymbol}{x}

\newcommand{\nonlinearity}{\phi}

\newcommand{\normal}{\mathcal{N}}

\newcommand{\expectation}[2]{\mathbb{E}_{#2}{\left\lbrack #1\right\rbrack}}

\newcommand{\sequenceVariable}{n}

\newcommand{\realLine}{\mathbb{R}}

\newcommand{\naturalNumbers}{\mathbb{N}}

\usepackage{thm-restate}

\newcommand{\datapoint}{x}

\newcommand{\indexSet}{\mathcal{X}}

\newcommand{\envelopegradient}{m}
\newcommand{\envelopeconstant}{c}

\newcommand{\genericRV}{X}

\newcommand{\rowIndex}{n}
\newcommand{\rowVariance}{\sigma^2_{\rowIndex}}

\newcommand{\limitStd}{\sigma_{*}}

\newcommand{\limitVariance}{\sigma^2_{*}}

\newcommand{\colIndex}{i}

\newcommand{\generalSum}{S}
\newcommand{\littleO}{\mathrm{o}}

\newcommand{\summand}{\gamma}

\newcommand{\projectionIndeces}{\mathcal{L}}
\newcommand{\projectionCoefficients}{\alpha}
\newcommand{\atWidth}{[\sequenceVariable]}

\newcommand{\monotoneCLTfunc}{h}

\usepackage{times}  
\usepackage{helvet} 
\usepackage{courier}  
\usepackage[hyphens]{url}  
\usepackage{graphicx} 
\usepackage{caption} 

\usepackage{amsthm}
\usepackage{amssymb}
\usepackage{amsmath}
\usepackage{multirow}
\usepackage[T1]{fontenc}
\usepackage[switch]{lineno} 
\newtheorem{prop}{Proposition}
\newtheorem{thm}{Theorem}
\newtheorem{lem}{Lemma}
\newtheorem{cor}{Corollary}
\newtheorem{defn}{Definition} 
\newtheorem{rem}{Remark}
\def\norm#1{\Vert#1\Vert}

\newcommand*\samethanks[1][\value{footnote}]{\footnotemark[#1]}

\title{On the Neural Tangent Kernel of Deep Networks with Orthogonal Initialization }
\author{%
	Wei~Huang \thanks{Both authors contributed equally to this work.}\\
	University of Technology Sydney, Australia \\
	\texttt{wei.huang-6@student.uts.edu.au} \\
	\And
	Weitao Du   \samethanks[1] \\
	Northwestern University, USA \\
	\texttt{weitao.du@northwestern.edu} \\
	\And
	Richard Yi Da Xu \\
	University of Technology Sydney, Australia \\
	\texttt{YiDa.Xu@uts.edu.au} \\
}

\begin{document}
	
\maketitle
	
	\begin{abstract}
		The prevailing thinking is that orthogonal weights are crucial to enforcing dynamical isometry and speeding up training. The increase in learning speed that results from orthogonal initialization in linear networks has been well-proven. However, while the same is believed to also hold for nonlinear networks when the dynamical isometry condition is satisfied, the training dynamics behind this contention have not been thoroughly explored. In this work, we study the dynamics of ultra-wide networks across a range of architectures, including Fully Connected Networks (FCNs) and Convolutional Neural Networks (CNNs) with orthogonal initialization via neural tangent kernel (NTK). Through a series of propositions and lemmas, we prove that two NTKs, one corresponding to Gaussian weights and one to orthogonal weights, are equal when the network width is infinite. Further, during training, the NTK of an orthogonally-initialized infinite-width network should theoretically remain constant. This suggests that the orthogonal initialization cannot speed up training in the NTK (lazy training) regime, contrary to the prevailing thoughts. In order to explore under what circumstances can orthogonality accelerate training, we conduct a thorough empirical investigation outside the NTK regime. We find that when the hyper-parameters are set to achieve a linear regime in nonlinear activation, orthogonal initialization can improve the learning speed with a large learning rate or large depth.
	\end{abstract}
	
	\section{Introduction}
	\label{intro}
	
	Deep learning has been responsible for a step-change in performance across machine learning, setting new benchmarks for state-of-the-art performance in many applications, from computer vision \cite{nixon2019feature}, natural language processing \cite{devlin2018bert}, to reinforcement learning \cite{mnih2015human}, and more. Beyond its fundamental shift in approach, an array of innovative techniques underpin the success of deep learning, such as residual connections \cite{he2016deep}, dropout \cite{srivastava2014dropout}, and batch normalization \cite{ioffe2015batch}. The mean field theory  \cite{poole2016exponential,schoenholz2016deep} recently opened a gate to analyze the principles behind neural networks with random, infinite width, and fully-connected networks as the first subjects. Broadly, what Schoenholz et al. \cite{schoenholz2016deep} discovered, and then empirically proved, is that there exists a critical initialization called the edge of chaos, allowing the correlation signal to go infinitely far forward and preventing vanishing or exploding gradients. Later, this theory had been extended to a much wider range of architectures, e.g., convolutional networks \cite{xiao2018dynamical}, recurrent networks \cite{chen2018dynamical}, dropout networks \cite{huang2019mean}, residual networks \cite{yang2017mean}, and batch normalization \cite{yang2019mean}.   
	
	Critical initialization requires the mean squared singular value of a network's input-output Jacobian to be $O(1)$. It was already known that the learning process in deep linear networks could be dramatically accelerated by ensuring all singular values of the Jacobian being concentrated near 1, a property known as {\it dynamical isometry} \cite{saxe2013exact}. However, what was not known was how to impose dynamical isometry in deep nonlinear networks. Pennington et al. \cite{pennington2017resurrecting,pennington2018emergence} conjectured that they could do so with techniques based on free probability and random matrix theory, giving rise to a new and improved form of initialization in deep nonlinear networks. Since then, dynamical isometry has been introduced to various architectures, such as residual networks \cite{tarnowski2018dynamical,ling2019spectrum}, convolutional networks \cite{xiao2018dynamical}, or recurrent networks \cite{chen2018dynamical} with excellent performance on real-world datasets. 
	
	In fully connected networks, two key factors help to ensure dynamical isometry. One is orthogonality, and the other is appropriately tuning weights' and biases' parameters to establish a linear regime in nonlinear activation \cite{pennington2017resurrecting}. In straightforward scenarios, orthogonal initialization is usually enough to impose dynamical isometry in a linear network. The benefit of the orthogonality in linear networks has been proven recently \cite{hu2020provable}. However, the dynamics of nonlinear networks with orthogonal initialization has not been investigated. The roadblock is that it has been unclear how to derive a simple analytic expression for the training dynamics. 
	
	Hence, to fill this gap, we look to a recent technique called neural tangent kernel (NTK) \cite{jacot2018neural}, developed for studying the evolution of a deep network using gradient descent in the infinite width limit. NTK is a kernel characterized by a derivative of the output of a network to its parameters. It has been shown that the NTK of a network with Gaussian initialization converges to a deterministic kernel and remains unchanged during gradient descent in the infinite-width limit. We extend these results to the orthogonal initialization case and find that orthogonal weights contribute to the same properties for NTK. Given a sufficiently small learning rate and wide width, the network optimized by gradient descent behaves as a model linearized about its initial parameters \cite{lee2019wide}, where these dynamics are called NTK regime, or {\it lazy training} \cite{chizat2019lazy}. As the learning rate gets larger or the network becomes deeper, that is, out of the NTK regime, we expect that there will be new phenomena that can differentiate two initialization. To summarize, our contribution is as follows,
	
	\begin{itemize}
\item We prove that the NTK of an orthogonally-initialized network converges to the NTK of a network initialized by Gaussian weights in the infinite-width limit. Besides, theoretically, during training, the NTK of an orthogonally-initialized infinite-width network stays constant in the infinite-width limit. 
\item We prove that the NTK of an orthogonally-initialized network across architectures, including FCNs and CNNs, varies at a rate of the same order for finite-width as the NTK of a Gaussian-initialized network. Therefore, there are no significant improvements brought by orthogonal initialization for wide and nonlinear networks compared with Gaussian initialization in the NTK regime.
\item We conduct a thorough empirical investigation of training speed outside the NTK regime to complement theoretical results. We show that orthogonal initialization can speed up training in the large learning rate and depth regime when the hyper-parameters are set to achieve a linear regime in nonlinear activation.
	\end{itemize}

	\section{Related Work}
	\label{rela}
	
	Hu et al. \cite{hu2020provable}'s investigation of orthogonal initialization in linear networks provided a rigorous proof that drawing the initial weights from the orthogonal group speeds up convergence relative to standard Gaussian initialization. However, deep nonlinear networks are much more complicated, making generating proof the same in these nonlinear settings much more difficult. For example,  Sokol and Park \cite{sokol2018information} attempted to explain why dynamical isometry imposed through orthogonal initialization can significantly increase training speed. They showed a connection between the maximum curvature of the optimization landscape, as measured by a Fisher information matrix (FIM) and the spectral radius of the input-output Jacobian, which partially explains why networks with greater isometric are able to train much faster. Given that NTK and FIM have the same spectrum for the regression problem, our theoretical results can be seen as a complement to each other since our work provide a precise characterization of the dynamics of wide networks in the NTK regime. At the same time, Sokol and Park \cite{sokol2018information} investigated the advantage of orthogonal initialization outside of the NTK regime.
	
	Jacot et al. \cite{jacot2018neural}, who conceived of the neural tangent kernel, shows that NTK both converges to an explicit limiting kernel in the infinite-width networks and remains constant during training with Gaussian initialization. Lee et al. \cite{lee2019wide} reached the same conclusion from a different angle with a demonstration that the gradient descent dynamics of the original neural network fall into its linearized dynamics regime. We extend these results to orthogonal initialization and show that both NTK of Gaussian and orthogonally-initialized network are the same kernel in the infinite-width limit. While the original work of NTK is groundbreaking in producing an equation to predict the behavior of gradient descent in the NTK regime, its proof is incomplete as it implicitly assumes gradient independence. Yang \cite{yang2019scaling,yang2020tensor} subsequently removed this assumption and completed the proof. Besides, Huang and Yau \cite{huang2019dynamics,arora2019exact,allen2019convergence,du2018gradient} have proven the same proprieties of NTK and global convergence of deep networks in a few different ways. However, all of these studies did not focus on the orthogonal initialization as with our work.
	
	Studies involving NTK commonly adopt the ntk-parameterization \cite{jacot2018neural}, since standard parameterization with an infinite-width limit can lead to divergent gradient flow in the infinite limit. 
	To account for this problem for standard parameterization, the authors of {sohl2020infinite} developed an improved standard parameterization. We tested both techniques using the ``Neural Tangents'' Python library  \cite{neuraltangents2019} with some added codes to support orthogonal initialization.
	
	\section{Preliminaries}
	\label{Preliminaries}
	
\subsection{Networks and Parameterization}
Suppose there are $D$ training points denoted by $\{ (x_d, y_d) \}_{d=1}^D$, where input $X = (x_1, \dots, x_D) \in \mathbb{R}^{n_{0} \times D} $, and label $Y = (y_1, \dots, y_D) \in \mathbb{R}^{n_{L} \times D} $. We consider the following architectures:

{\bf Fully-Connected Network (FCN).} Consider a fully-connected network of $L$ layers of widths $n_l$, for $l = 0, \cdots, L$, where $l=0$ is the input layer and $l=L$ is output layer.  Following the typical nomenclature of literature, we denote synaptic weight and bias for the $l$-th layer by $W^l \in \mathbb{R}^{n_{l} \times n_{l-1}} $ and $b^l \in \mathbb{R}^{n_l} $, with a point-wise activations function $\phi: \mathbb{R} \rightarrow \mathbb{R} $.
 For each input $x \in \mathbb{R}^{n_0}$, pre-activations and post-activations are denoted by $h^l(x) \in \mathbb{R}^{n_l}$ and $x^l(x) \in \mathbb{R}^{n_l}$ respectively. The information propagation for $l \in \{1,\dots,L\}$ in this network is govern by,
  \begin{equation} \label{eq:fcn}
    x_i^l  = \phi(h_i^l),~~~ h^l_i =  \sum_{j=1}^{n_l} W_{ij}^l x_j^{l-1} + b^l_i,
  \end{equation}

{\bf Convolutional Neural Network (CNN).} For notational simplicity, we consider a 1D convolutional networks with periodic boundary conditions. We denote the filter relative spatial location $\beta \in  \{-k,\dots, 0, \dots, k \}$ and spatial location $\alpha \in \{1, \dots, m \}$, where $m$ is the spatial size. The forward propagation  for $l \in \{1,\dots,L-1 \}$ is given by, 
\begin{equation} \label{eq:cnn}
    x_{i,\alpha}^l  = \phi(h_{i,\alpha}^l),~~~ h_{i,\alpha}^l = \sum_{j=1}^{n_l} \sum_{\beta=-k}^{k} W_{ij,\beta}^l x_{j,\alpha+\beta}^{l-1} + b_i^l,
  \end{equation}
where weight $W^l \in \mathbb{R}^{n_{l} \times n_{l-1} \times (2k+1)} $, and $n_{l}$ is the number of channels in the $l^{th}$ layer. The output layer of a CNN is processed with a fully-connected layer, $f_i(x) = h^{L}_i = \sum_{j=1}^{n_{L}} \sum_{\alpha} W^{L}_{ij,\alpha} x_{j,\alpha}^{L-1}$.

Standard parameterization requires the parameter set $\theta = \{ W^l_{ij}, b^l_i\}$ is an ensemble generated by, $W^l_{ij} \sim  \mathcal{N} (0, \frac{\sigma^2_w}{ n_{l-1}}) , ~~~ b_i^l \sim \mathcal{N} (0,\sigma^2_b)$, where $\sigma^2_w$ and $\sigma^2_b$ are weight and bias variances. The variance of weights is scaled by the width of previous layer $n_{l-1}$ to preserve the order of post-activations layer to be $O(1)$. We denote this parameterizationas {\it standard parameterizaiton}. However, this paramterization can lead to a divergence in derivation of neural tangent kernel. To overcome this problem, {\it ntk-parameterization} was introduced, $W^l_{ij} =  \frac{\sigma_w}{\sqrt{n_{l-1}}} \omega^l_{ij}   , ~~~ b_i^l = \sigma_b \beta_i^l$, where $ \omega^l_{ij}, \beta_i^l  \sim \mathcal{N} (0,1)$. 

\begin{table*}[!htbp]
		\begin{center}
			\begin{tabular}{ | c |c | c| c|}
			\hline   
			  Parameterization  &  $W$ initialization           & $b$ initialization                    &  layer equation   \\
			 \hline 
			  ntk Gaussian  & $ \omega_{ij} \sim \mathcal{N} (0,1)$   & \multirow{2}{*}{$\beta_i \sim \mathcal{N}(0,1)$ } & \multirow{2}{*}{$h^l=  \frac{\sigma_w}{\sqrt{n_{l-1}}}\omega^l x^{l-1}+\sigma_b \beta^l$} \\
                                ntk Orthogonal      & $(\omega^l)^T \omega^l = n_{l-1}{\bf I}$             &          &    \\
		      std Gaussian       & $W_{ij} \sim \mathcal{N}(0,\frac{\sigma_w^2}{N_{l-1}})$  &\multirow{2}{*}{$b_i \sim \mathcal{N}(0,\sigma_b^2)$}   & \multirow{2}{*}{$h^l=\frac{1}{\sqrt{s}} W^l x^{l-1}+ b^l$}   \\
             std Orthogonal     & $W^T W = \sigma_w^2{\bf I}$             &    &   \\       
          \hline 
 ntk Gaussian      & $ \omega_{ij,\alpha} \sim \mathcal{N} (0,1)$   & \multirow{2}{*}{$\beta_i \sim \mathcal{N}(0,1)$ } & \multirow{2}{*}{$h^l_\alpha=  \sum_{\beta=-k}^{k} \frac{\sigma_w}{\sqrt{(2k+1)n_{l-1}}}\omega_\beta^l x_{\alpha+\beta}^{l-1}+\sigma_b \beta^l$} \\
                                ntk Orthogonal      & $(\omega_\alpha^l)^T \omega_\alpha^l = n_{l-1}{\bf I}$             &          &    \\
		      std Gaussian       & $W_{ij,\alpha} \sim \mathcal{N}(0,\frac{\sigma_w^2}{(2k+1)N_{l-1}})$  &\multirow{2}{*}{$b_i \sim \mathcal{N}(0,\sigma_b^2)$}   & \multirow{2}{*}{$h_\alpha^l=\frac{1}{\sqrt{s}} W_\beta^l x_{\alpha+\beta}^{l-1}+ b^l$}   \\
             std Orthogonal     & $W_\alpha^T W_\alpha = \frac{\sigma_w^2}{2k+1}{\bf I}$             &    &   \\       
          \hline 
			\end{tabular}
		\end{center}
    \caption{Summary of improved standard parameterization and ntk-parameterization for Gaussian and orthogonal initialization. The abbreviation ``std'' stands for standard, and the ``parameterization'' is omitted after ntk or std.}
\label{tab:parameterization}
	\end{table*}	
	
	\subsection{Dynamical Isometry and Orthogonal Initialization}

Consider the input-output Jacobian $J = \frac{\partial h^{L}}{\partial x^{0}} = \prod_{l=1}^L D^l W^l$, where $h^{L}$ is output function, $x^{0}$ is input, and $D^l$ is a diagonal matrix with elements $D^l_{ij}=\phi'(h^l_i)\delta_{ij}$. Ensuring all singular values of the Jacobian concentrate near 1 is a property known as {\it dynamical isometry}. In particular,  It is shown that two conditions regarding singular values of $W^l$ and $D^l$ contribute crucially to the dynamical isometry in non-linear networks \cite{pennington2017resurrecting}. More precisely, the singular values of $D^l$ can be made arbitrarily close to 1 by choosing a linear regime in a nonlinear activation. On the other hand, adopting a random orthogonal initialization can force the singular values of weights into 1. In particular, weights are drawn from a uniform distribution over scaled orthogonal matrices obeying,
 \begin{equation} \label{eq:gaussian}
    (W^l)^T  W^l = \sigma^2_w {\bf I}, 
  \end{equation}
This is the {\it standard parameterization} for orthogonal weights, and {\it ntk-parameterization} of orthogonality follows,
 \begin{equation} \label{eq:gaussian}
     W^l_{ij} =  \frac{\sigma_w}{\sqrt{n_{l-1}}} \omega^l_{ij}, ~~~(\omega^l)^T \omega^l = n_{l-1}{\bf I}.
  \end{equation}
We show a summary of improved standard parameterization and ntk-parameterization across FCN and CNN for Gaussian and orthogonal initialization in Table \ref{tab:parameterization}. The factor $s$ in the layer equation of standard parameterization is introduced to prevent divergence of NTK \cite{sohl2020infinite}. The core idea is to write the width of the neural network in each layer in terms of an auxiliary parameter, $n_l = s N_l$. Instead of letting $n_l \rightarrow \infty$, we adopt $s$ as the limiting factor.

\subsection{Neural Tangent Kernel}

The neural tangent kernel (NTK) is originated from \cite{jacot2018neural} and defined as,
 \begin{equation}\label{eq:Jacobian}
   \Theta_t(X,X) = \nabla_\theta f_{t}(\theta,X) \nabla_\theta f_t(\theta,X)^T.
 \end{equation}
where function $f_t$ are the outputs of the network at training time $t$, i.e. $f_t(\theta,X)= h_t^L(\theta,X) \in \mathbb{R}^{D \times n_L} $, and $ \nabla_\theta f_t(\theta,X) = {\rm vec}([\nabla_\theta f_t(\theta,x)]_{x \in X}) \in \mathbb{R}^{D n_L} $. As such, the neural tangent kernel is formulated as a $D n_L \times D n_L $ matrix. Let $\eta$ be the learning rate, and $\mathcal{L}$ be the loss function. The ynamics of gradient flow for parameters and output function are given by,
\begin{equation}
 \begin{aligned}
\frac{\partial \theta}{\partial t} &=   - \eta \nabla_\theta \mathcal{L} = - \eta  \nabla_\theta f_t (\theta,X)^T \nabla_{f_t(\theta,X)} \mathcal{L} \\
\frac{\partial f_t(\theta, X)}{\partial t}& =  \nabla_\theta f_t(\theta, X) \frac{\partial \theta}{\partial t}   = - \eta \Theta_t(X,X) \nabla_{f_t(\theta,X)} \mathcal{L}. \\
\end{aligned}
\end{equation}
This equation for $f_t$ has no substantial insight in studying the behavior of networks because $\Theta_t(X,X)$ varies with the time during training. Interestingly, as shown by \cite{jacot2018neural}, the NTK $\Theta_t(X,X)$ converges to a deterministic kernel $\Theta_\infty(X,X)$ and does not change during training in the infinite-width limit, i.e. $\Theta_t(X,X) = \Theta_\infty(X,X)$. As a result, the infinite width limit of the training dynamics are given by,
 \begin{equation}\label{eq:ode}
   \frac{\partial f_t(\theta, X)}{\partial t} =  - \eta \Theta_\infty(X,X) \nabla_{f_t(\theta,X)} \mathcal{L}.
 \end{equation}
In the case of an MSE loss, $\mathcal{L}(y,f) = \frac{1}{2}   \left\lVert y-f_t(\theta,x)\right\lVert^2_2$, the Equation (\ref{eq:ode}) becomes a linear model with a solution,
\begin{equation}\label{eq:Jacobian}
   f_t(\theta, X) = ( {\bf I} - e^{ - \eta \Theta_\infty(X,X)t} ) Y +  e^{ - \eta \Theta_\infty(X,X)t} f_0(\theta,X).
 \end{equation}
	
\section{Theoretical Results}
\label{Theory}

\subsection{An Orthogonally Initialized Network is a Gaussian Process in the Infinite Width Limit}

As stated in \cite{lee2017deep,matthews2018gaussian}, the pre-activation $h^l_i$ of Gaussian initialized network tends to Gaussian processes (GPs) in the infinite-width limit. This is the proposition to construct the NTK in networks with Gaussian weights \cite{jacot2018neural}. We extend this result to the orthogonal initialization across Fully Connected Networks (FCNs) and Convolutional Neural Networks (CNNs):
\begin{thm}\label{thm:output_limit}
Consider a FCN of the form (\ref{eq:fcn}) at orthogonal initialization, with a Lipschitz nonlinearity $\phi$, and in the limit as $n_1, ..., n_{L-1} \to \infty$, the pre-activations $h^l_{i}$, for $i=1, ..., n_{l}$ and $l \in \{1,\dots, L\}$, tend to i.i.d centered Gaussian processes of covariance $\Sigma^{l}$ which is defined recursively by:
\begin{align*}
\Sigma^{1}(x, x') &= \frac{\sigma^2_w}{n_0} x^T x' + \sigma^2_b \\
\Sigma^{l}(x, x') &= \sigma^2_w \mathbb{E}_{f\sim\mathcal{N}\left(0,\Sigma^{l-1}\right)}[\phi(f(x)) \phi(f(x'))] + \sigma^2_b,
\end{align*}

For a CNN of the form (\ref{eq:cnn}) at orthogonal initialization, and in the limit as $n_1, ..., n_{L-1} \to \infty$, the pre-activations $h^l_{i,\alpha}$ tend to Gaussian processes of covariance $\Sigma_{\alpha,\alpha'}^{l}$, which is defined recursively by:
	\begin{align*}
	&\Sigma^{1}_{\alpha,\alpha'}(x, x') = \frac{\sigma^2_w}{n_0(2k+1)} \sum_{\beta=-k}^k  x_{\alpha+\beta}^T x'_{\alpha'+\beta} + \sigma^2_b \\
	&\Sigma^{l}_{\alpha,\alpha'}(x, x') = \frac{\sigma^2_w}{(2k+1)} \sum_{\beta=-k}^k \big[ \mathbb{E}_{f\sim\mathcal{N}\left(0,\Sigma_{\alpha+\beta,\alpha'+\beta}^{l-1}\right)}\\
	&   ~~~~~~~~~~~~~     [\phi(f(x_{\alpha+\beta})) \phi(f(x_{\alpha'+\beta}'))] \big] + \sigma^2_b. \\
   & \Sigma^{L}(x, x')  =  \sum_\alpha \delta_{\alpha,\alpha'} \big[ \mathbb{E}_{f\sim\mathcal{N}\left(0,\Sigma_{\alpha,\alpha'}^{L-1}\right)}  [\phi(f(x_{\alpha})) \phi(f(x_{\alpha'}'))] \big]
	\end{align*}
\end{thm}
Different from the independence property of Gaussian initialization, the entries of the orthogonal matrix are correlated. We use the Stein method and exchangeable sequence to overcome this difficulty and leave the detailed proof in the appendix. As shown by Theorem \ref{thm:output_limit}, neural networks with Gaussian and orthogonal initialization are in correspondence with an identical class of GPs.

\subsection{Neural Tangent Kernel at Initialization}

\begin{figure*}[t!]
\centering
  \centering
  \includegraphics[width=0.7\textwidth]{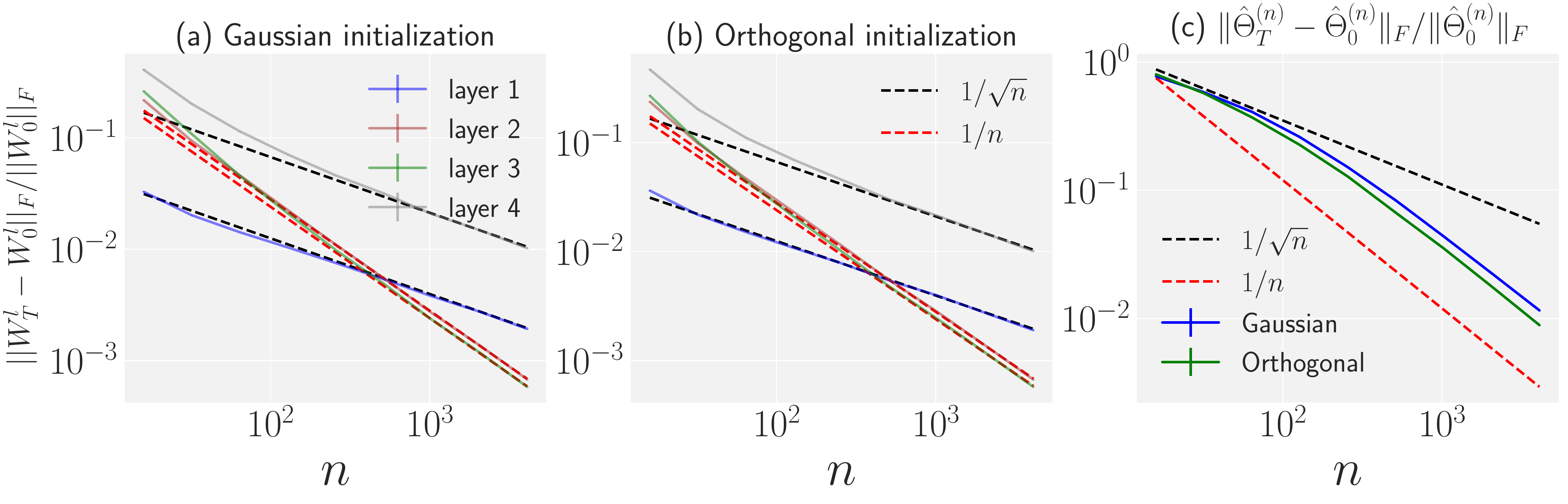}
 \caption{Changes of weights, empirical NTK on a three hidden layer Erf Network. Solid lines correspond to empirical simulation, and dotted lines are theoretical predictions, i.e., black dotted lines are $1/\sqrt{n}$ while red dotted lines are $1/n$. (a) Weight changes on the Gaussian initialized network. (b) Weight changes on the orthogonally initialized network. (c) NTK changes on networks with Gaussian and orthogonal initialization. }
 \label{fig:1}
\end{figure*}

According to \cite{jacot2018neural}, the NTK of a network with Gaussian weights converges in probability to a deterministic kernel in the infinite-width limit. We show that the NTK of an orthogonally initialized network is identical to the one with Gaussian weights in the infinite-width limit. 
\begin{thm}\label{thm:convergence_NTK_initialization}
Consider a FCN of the form (\ref{eq:fcn}) at orthogonal initialization, with a Lipschitz nonlinearity $\phi$, and in the limit as the layers width $n_1, ..., n_{L-1} \to \infty$, the NTK $\Theta^L_0(x,x')$, converges in probability to a deterministic limiting kernel: $$\Theta_0^{L}(x,x') \to \Theta^{L}_\infty(x,x') \otimes {\bf I}_{n_L \times n_L}.$$
The scalar kernel $\Theta^{L}_\infty(x,x')$ is defined recursively by
\begin{align*}
    \Theta^{1}_\infty(x, x') &= \Sigma^{1}(x, x') \\
    \Theta^{l}_\infty(x, x') &= \sigma^2_w  \dot{\Sigma}^{l}(x, x') \Theta^{l-1}_\infty(x, x')  + \Sigma^{l}(x, x'),
\end{align*}
where
\[
	\dot{\Sigma}^{l}\left(x,x'\right)=
	\mathbb{E}_{f\sim\mathcal{N}\left(0,\Sigma^{\left(l-1\right)}\right)}\left[\dot{\phi}\left(f\left(x\right)\right)\dot{\phi}\left(f\left(x'\right)\right)\right],
\]

For a CNN of the form (\ref{eq:cnn}) at orthogonal initialization, and in the limit as $n_1, ..., n_{L-1} \to \infty$, the NTK $\Theta^{L}_0(x,x')$, converges in probability to a deterministic limiting kernel: $$\Theta^{L}_0(x,x') \to \Theta^{L}_\infty(x,x') \otimes {\bf I}_{n_L \times n_L}.$$
The scalar kernel $\Theta^{L}_\infty(x,x')$ is given recursively by
\begin{align*}
   & {\Theta_{\alpha,\alpha'}^{1}}_\infty(x, x') = \Sigma_{\alpha,\alpha'}^{1}(x, x') \\
   & {\Theta_{\alpha,\alpha'}^{l}}_\infty(x, x') = \frac{\sigma^2_w}{(2k+1)} \sum_\beta \big[ \dot{\Sigma}_{\alpha+\beta,\alpha'+\beta}^{l}(x, x')\\
                                                &~~~~~~~~~~~~~ {\Theta_{\alpha+\beta,\alpha'+\beta}^{l-1}}_\infty(x, x')  + \Sigma_{\alpha+\beta,\alpha'+\beta}^{l}(x, x') \big] \\
   & \Theta^L_\infty(x,x')  = \sum_\alpha \delta_{\alpha,\alpha'} \big[  \dot{\Sigma}_{\alpha,\alpha'}^{L}(x, x') 
                                                 {\Theta_{\alpha,\alpha'}^{L-1}}_\infty(x, x') \\
                                            &~~~~~~~~~~~~~~~ + \Sigma_{\alpha,\alpha'}^{L}(x, x')   \big]
\end{align*}
\end{thm}

\begin{rem}
Since the Lipschitz function is differentiable besides a measure zero set, then taking the expectation would not destroy the whole statement, which allows for the ReLU activation. 
\end{rem}
 	
In general, the NTK of CNNs can be computed recursively in a similar manner to the NTK for FCNs. However, the NTK of CNNs propagate differently by averaging over the NTKs regarding the neuron location of the previous layer. According to Theorem \ref{thm:convergence_NTK_initialization}, the NTK of an orthogonally initialized network converges to an identical kernel as Gaussian initialization. This suggests these two NTKs are equivalent when the network structure (depth of $L$, filter size of $2k+1$, and activation of $\phi$) and choice of hyper-parameters ($\sigma^2_w$ and $\sigma^2_b$) are the same in the infinite-width limit.


\begin{figure*}[t!]
\centering
  \centering
  \includegraphics[width=0.8\textwidth]{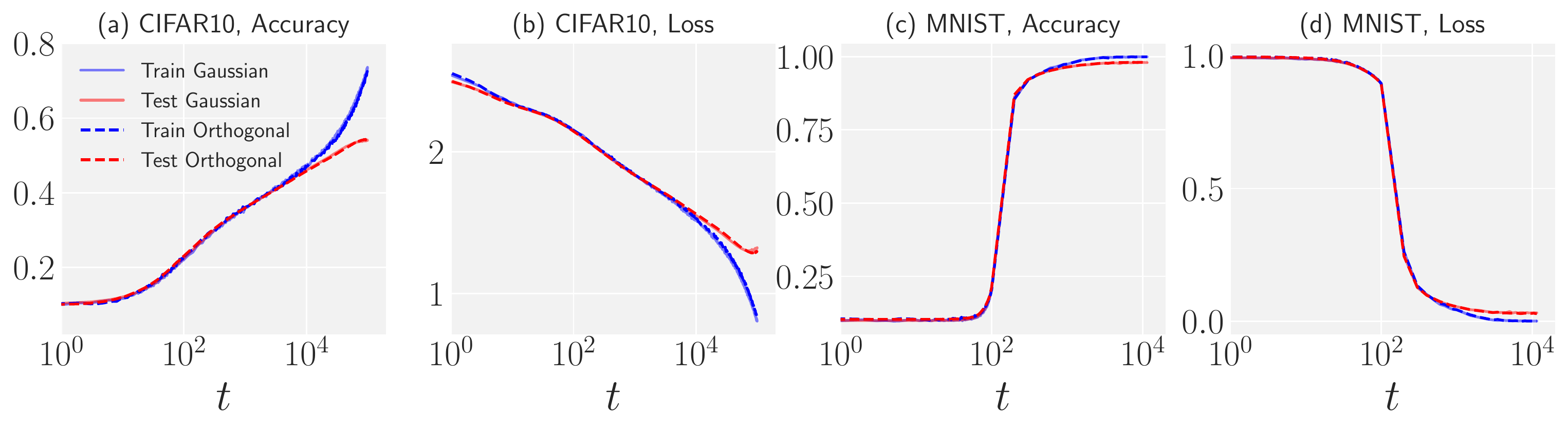}
 \caption{ Orthogonally initialized networks behave similarly to the networks with Gaussian initialization in the NTK regime. (a,b) We adopt the network architecture of depth of $L=5$, width of $n = 800$, activation of tanh function, with $\sigma^2_w = 2.0$, and $\sigma^2_b =0.1$. The networks are trained by SGD with a small learning rate of $\eta = 10^{-3}$ for $T=10^5$ steps with a batch size of $10^3$ on cross-entropy loss on full CIFAR-10. (c,d)  We adopt the network architecture of depth of $L=9$, width of $n = 1600$, activation of ReLU function, with $\sigma^2_w = 2.0$, and $\sigma^2_b =0.1$. The networks are trained by PMSProp with a small learning rate of $\eta = 10^{-5}$ for $T=1.2 \times 10^4$ steps with a batch size of $10^3$ on MSE loss on MNIST. While the solid lines stand for Gaussian weights, dotted lines represent orthogonal initialization.}
 \label{fig:2}
\end{figure*}

\subsection{Neural Tangent Kernel During Training}

It is shown that the NTK of a network with Gaussian initialization stays asymptotically constant during gradient descent training in the infinite-width limit, providing a guarantee for loss convergence \cite{jacot2018neural}. We find that the NTK of orthogonally initialized networks have the same property, which is demonstrated below in an asymptotic way,
 \begin{thm}\label{thm:main}
            Assume that  $\lambda_{\rm min}(\Theta_\infty) >0$ and $\eta_{\rm critical} = \frac{\lambda_{\rm min}(\Theta_\infty)+\lambda_{\rm max}(\Theta_\infty)}{2}$. Let $n = n_1, ..., n_{L-1}$ be the width of hidden layers. Consider a FCN of the form (\ref{eq:fcn}) at orthogonal initialization, trained by gradient descent with learning rate $\eta < \eta_{\rm critical}$ (or gradient flow). For every input $x\in \mathbb R^{n_0}$ with $\|x\|_2\leq 1$, with probability arbitrarily close to 1,  
           \begin{equation}
            \sup_{t\geq 0}\frac{\left\|\theta_t -\theta_0\right\|_2}{\sqrt n},
            \,\, 
            \sup_{t\geq 0}\left\|\hat \Theta_t - \hat \Theta_0 \right\|_F = O(n^{-\frac 1 2}), \,\, {\rm as }\quad n\to \infty\,. 
           \end{equation}
 where $\hat \Theta_t$  are empirical kernels of networks with finite width.
 
 For a CNN of the form (\ref{eq:cnn}) at orthogonal initialization, trained by gradient descent with learning rate $\eta < \eta_{\rm critical}$ (or gradient flow), for every input $x\in \mathbb R^{n_0}$ with $\|x\|_2\leq 1$, and filter relative spatial location $\beta \in  \{-k,\dots, 0, \dots, k \}$, with probability arbitrarily close to 1,
           \begin{equation}
            \sup_{t\geq 0}\frac{\left\|\theta_{\beta,t} -\theta_{\beta,0} \right\|_2}{\sqrt n},
            \,\, 
            \sup_{t\geq 0}\left\| \hat \Theta_t -   \hat \Theta_0 \right\|_F = O(n^{-\frac 1 2}).
           \end{equation}
 \end{thm} 
   
\cite{jacot2018neural} proved the stability of NTK under the assumption of global convergence of neural networks, while \cite{lee2019wide} provided a self-contained proof of both global convergence and stability of NTK simultaneously. In this work, we refer to the proof strategy from \cite{lee2019wide,liu2020linearity} and extend it to the orthogonal case, as shown in the appendix.

To certificate this theorem empirically, we adopt three hidden layers Erf networks trained by gradient descent with learning rate $\eta =1.0$ on a subset of the MNIST dataset of $D = 20$. We measure changes of weights, empirical NTK after $T = 2^{15}$ steps of gradient descent for varying width at both Gaussian and orthogonal initialization. Figure \ref{fig:1}(a,b) show that the relative change in the first and last layer weights scales as $1/\sqrt{n}$ while second and third layer weights scale as $1/n$ with Gaussian and orthogonal weights respectively. In Figure \ref{fig:2}(c), we observe the change in NTK is upper bounded by $O(1/\sqrt{n})$ but is closer to $O(1/n)$ for both Gaussian and orthogonal initialization. The discrepancy between theoretical bound ($O(n^{-1/2})$) and experimental observation ($O(n^{-1})$) has been solved in \cite{huang2019dynamics}, where they prove that relative change of empirical NTK of Gaussian initialized networks is bounded by $O(1/n)$. Without loss of generality, we infer that the proof framework is suitable for orthogonal weights.

\begin{figure*}[t!]
\centering
  \centering
  \includegraphics[width=0.85\textwidth]{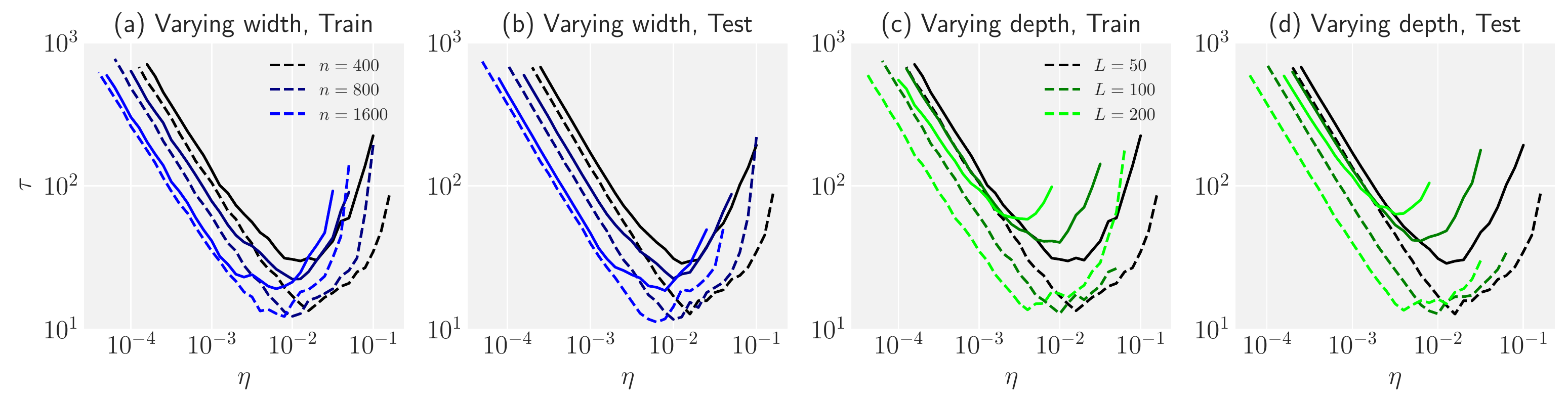}
 \caption{The steps $\tau$ as a function of learning rate $\eta$ of two lines of networks on both train and test dataset. The results of orthogonal networks are marked by dotted lines while those of Gaussian initialization are plotted by solid lines. Networks with varying width, i.e. $n = 400, 800,$ and $1600$, on (a) train set and (b) test set;  Networks with varying depth, i.e. $L = 50, 100,$ and $200$, on (c) train set and (d) test set. Different colors represent the corresponding width and depth. While curves of orthogonal initialization are lower than those of Gaussian initialization in the small learning rate phase, the differences become more significant in the large learning rate. Besides, the greater the depth of the network, the more significant the difference in performance between orthogonal and Gaussian initialization.}
 \label{fig:3}
\end{figure*}

\begin{figure}[t!]
\centering
  \centering
  \includegraphics[width=0.75\textwidth]{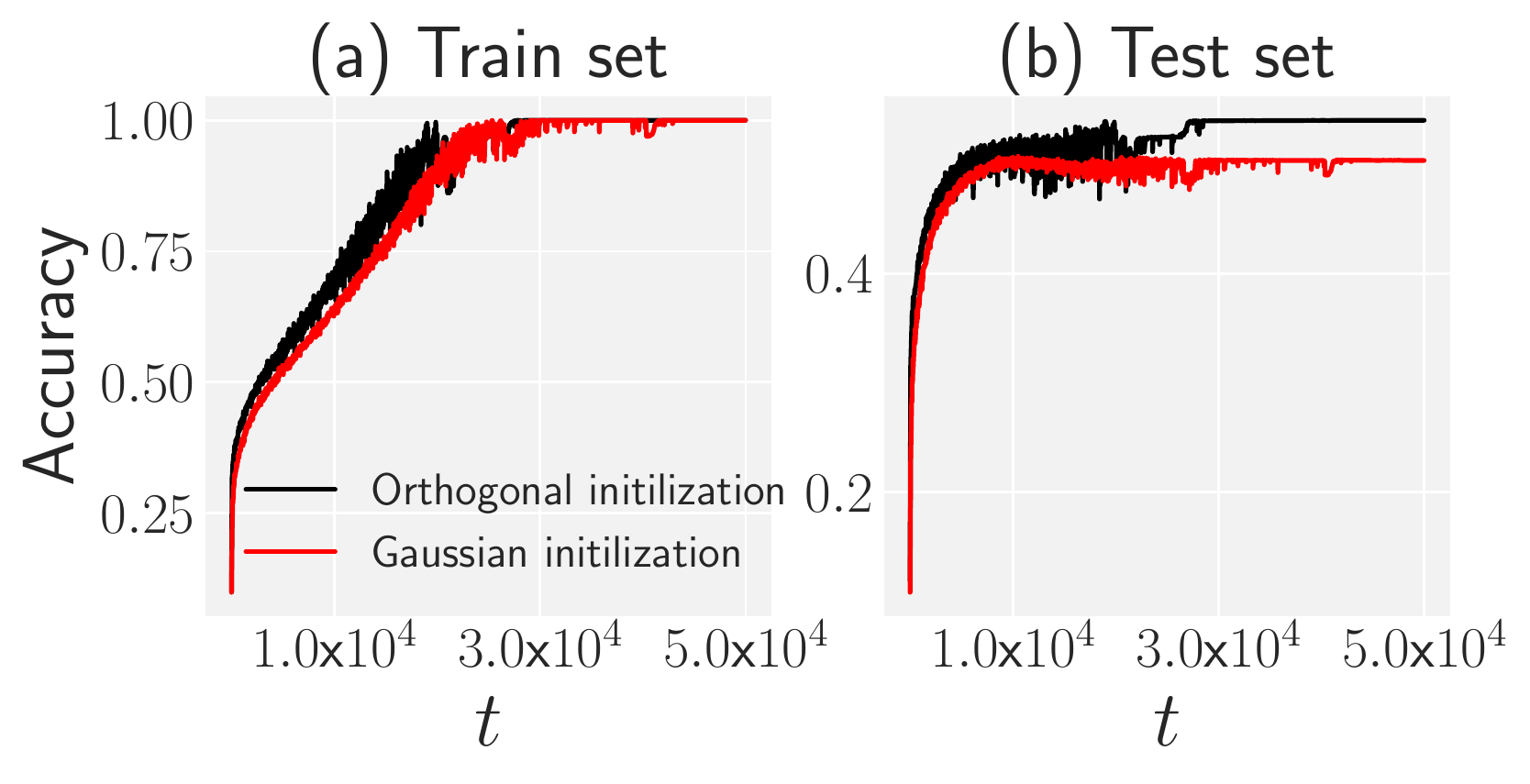}
 \caption{Learning dynamics measured by the optimization and generalization accuracy on train set and test set, for networks of depth $L=100$ and width $n=400$. We additionally average our results over 30 different instantiations of the network to reduce noise. Black curves are the results of orthogonal initialization, and red curves are performances of Gaussian initialization. (a) The training speed of an orthogonally initialized network is faster than that of a Gaussian initialized network. (b) On the test set, compared to the network with Gaussian initialization, the orthogonally initialized network not only learns faster but ultimately converges to a higher generalization performance.}
 \label{fig:4}
\end{figure}

\section{Numerical Experiments}

Our theoretical result indicates that ultra-wide networks with Gaussian and orthogonal initialization should have the same convergence rate during the gradient descent training. This means that two different initializations have similar training dynamics for loss and accuracy function in the NTK regime. Thus, it is now for us to verify our theories in practice. To this end, we perform a series of experiments on MNIST and CIFAR10 dataset. All the experiments are performed with the standard parameterization with TensorFlow.

We compare the train and test loss and accuracy with two different initialization, i.e., Gaussian and orthogonal weights using $D=256$ samples on full CIFAR-10 and MNIST dataset, as summarized in Figure \ref{fig:2}. To reduce noise, we averaged the results over 30 different instantiations of the networks. Figures \ref{fig:2}(a,b) show the results of the experiments on networks of depth $L=5$, width $n=800$, and activation tanh function, using SGD optimizer with a small learning rate of $\eta = 10^{-3}$ for $T=10^5$ steps on CIFAR-10 dataset. Figure \ref{fig:2}(c)(d) display the results on networks of depth $L=9$, width $n=1600$, and activation ReLU function, using PMSProp \cite{hinton2012neural} optimizer with a small learning rate of $\eta = 10^{-5}$ for $T=1.2 \times 10^4$ steps on MNIST. In all cases, we see an excellent agreement between the training dynamics of the two initialization, which is consistent with our theoretical finding (Theorem \ref{thm:main}).



Having confirmed the consistency between training speed of networks with Gaussian and orthogonal initialization in the NTK regime, our primary interest is to find when orthogonal initialization accelerates the training speed for nonlinear networks. We need to go beyond the NTK regime and experiment with an additional requirement for hyper-parameters according to the evidence that orthogonal initialization increases learning speeds when the variance of weights and biases is set to achieve a liner regime in nonlinear activation \cite{pennington2017resurrecting}. 

Following \cite{pennington2017resurrecting}, we set $\sigma^2_w = 1.05$, and $\sigma^2_b = 2.01 \times 10^{-5}$, and $\phi(x) = {\rm tanh}(x)$. We then vary the width of network in one set of experiments as $n=400,800$ and $1600$ when $L=50$, and the depth in another as $L=50,100$ and $200$, when $n=400$. All networks are trained by SGD optimizer on CIFAR-10 dataset. To evaluate the relationship between the learning rate and training speed, we select a threshold accuracy of $p = 0.25$ and measure the first step $\tau$ when accuracy exceeds $p$. Figure \ref{fig:3} shows the steps of $\tau$ as a function of the learning rate of $\eta$ for both the training and testing sets.

The results in Figure \ref{fig:3} suggest a more quantitative analysis of the learning process until convergence. We train networks listed in Figure \ref{fig:3} for $5 \times 10^4$ steps with a certain learning rate. We show the results of a certain network of depth $L=100$ and width $n=400$ trained with a learning rate $\eta = 0.01$ as a typical example in Figure \ref{fig:4}. The results of other network structures can be found in the appendix. It is shown that the training speed of orthogonally initialized networks is faster than that of Gaussian initialized networks {\it outside} the NTK regime. At the same time, orthogonally initialized networks can finally obtain a higher generalization result.

We draw two main conclusions from these experiments. First, orthogonal initialization results in faster training speeds and better generalization than Gaussian initialization in the large learning rate phase. It was shown that the large learning rate phase has many different properties from the small learning rate phase \cite{lewkowycz2020large,li2019towards}. Our finding can be seen as another effect in the large learning rate phase. Second, given the constant width, the greater the depth of the network, the more significant the difference in performance between orthogonal and Gaussian initialization. This phenomenon is consistent with the theoretical result observed in deep linear networks. It was found that the width needed for efficient convergence to a global minimum with orthogonal initialization is independent of the depth. In contrast, the width needed for efficient convergence with Gaussian initialization scales proportionally in depth \cite{hu2020provable}.

\section{Conclusion}
	
This study on the neural tangent kernel of wide and nonlinear networks with orthogonal initialization has proven, theoretically and empirically, that the NTK of an orthogonally-initialized network across both FCN and CNN converges to the same deterministic kernel of a network initialized from Gaussian weights in the finite-width limit. We find that with an infinite-width network and a gradient descent (gradient flow) training scheme, the NTK of an orthogonally initialized network does not change during training. Further, it has the same order convergence rate from a finite to an infinite width limit as that of a Gaussian initialized network. Our theoretical results suggest that the dynamics of wide networks with orthogonal initialization behave similarly to that of Gaussian networks with a small learning rate verified by experiments. This observation implies that orthogonal initialization is only effective when not in the lazy (NTK) regime. And it is consistent with the fact that the infinite-width analysis does not explain the practically observed power of deep learning \cite{arora2019exact,chizat2019lazy}. Last, we find that orthogonal networks can outperform Gaussian networks in the large learning rate and depth on both train and test sets. 
	
\setcounter{prop}{0}
\setcounter{thm}{0}
\setcounter{lem}{0}
\setcounter{rem}{0}

\appendix
\section{Appendix}
This appendix is dedicated to proving the key results of this paper, namely Theorem \ref{thm:output_limit}, Theorem \ref{thm:convergence_NTK_initialization}, and Theorem \ref{thm:main} based on a series of lemmas, which describe the asymptotics of neural networks with orthogonal weights at initialization and during training. We prove the Gaussian process behavior of the output function, the limiting deterministic kernel for the orthogonal initialization, and its stability during training in the first three sections. Besides, we provide more experiments in the NTK regime and outside the NTK regime in the last section.

\subsection{NNGP at Initialization}

Throughout this section, we assume that $n_1 = n_2 = \cdots = n_L = n$. That is, $\{W_{ij}\}_{n\times n}$ is an orthogonal mapping. We first cite the following lemma from \cite{chatterjee2007multivariate}, which describes how the post-activations of one-layer transform by multiplying a random orthogonal matrix.
\begin{lem}\label{lem:haar}
Let $(W_{ij})_{n\times n}$ be an orthogonal matrix randomly sampled by the Haar measure of orthogonal matirx. Let $B$ be a $n \times n$ matrix s.t $Tr(BB^T) = n$. Then $Tr(BW)$ converges to a standard Gaussian distribution as the size of the matrix $n$ tends to infinity.
\end{lem} 
If we condition on the previous layer's output, the next layer's pre-activations are Gaussian when the width tends to infinity. Therefore if we take the limit of previous layers $n_1,\dots,n_{l-1} \rightarrow \infty$ sequentially, the pre-activation $\{h_i^l\}$ of the l-th layer tends to a centered Gaussian process with respect to the input vector x. However, our goal is to take all the previous layers' width simultaneously. The main technical difficulty is that we lose the independence between different index $i$ in the finite width when we implement orthogonal ensemble. Hence analysis based on the central limit theorem for i.i.d random sequences would not work. Instead, we follow the strategy in \cite{matthews2018gaussian} to apply a modified version of the exchangeable random sequence central limit theorem. Note that the Haar probability measure is invariant under row and column permutations. This implies that permuting the output index $i$ won't change the joint law of $\{h_i^l (x)\}_{1 \leq i \leq n_l}$. 

We will use the following adapted version of the central limit theorem for exchangeable sequences in \cite{matthews2018gaussian}.

\begin{lem}\label{lemma:gclta}
For each positive integer $\rowIndex$ let $\left( \genericRV_{\rowIndex, \colIndex} ; \colIndex = 1,2,... \right)$ be an infinitely exchangeable process with zero mean, finite variance $\rowVariance$, and finite absolute third moment.  Suppose also that the variance has a limit $\lim_{\rowIndex \to \infty}\rowVariance = \limitVariance$. Define

\begin{equation}
\generalSum_{\rowIndex} = \frac{1}{\sqrt{\monotoneCLTfunc(\rowIndex)}}\sum_{\colIndex=1}^{\monotoneCLTfunc(\rowIndex)} \genericRV_{\rowIndex,\colIndex} \, , 
\end{equation}

\noindent where $\monotoneCLTfunc : \naturalNumbers \mapsto \naturalNumbers$ is a strictly increasing function. If the following conditions hold:

\begin{enumerate}[a)]
	\item $\expectation{\genericRV_{\rowIndex, 1} \genericRV_{\rowIndex, 2}}{\rowIndex} = \littleO(\frac{1}{h(n)}) $
	\item $ \lim_{\rowIndex \to \infty }\expectation{\genericRV_{\rowIndex, 1}^{2} \genericRV_{\rowIndex, 2}^{2}}{\rowIndex} = \limitStd^{4} $
	\item $ \expectation{|\genericRV_{\rowIndex, 1}|^{3}}{\rowIndex} = \littleO(\sqrt{\monotoneCLTfunc(\rowIndex)}) $
\end{enumerate}

Then $\generalSum_{\rowIndex}$ converges in distribution to $\normal(0,\limitVariance)$.

\end{lem}

As it was pointed out in \cite{matthews2018gaussian}, convergence with respect to arbitary finite-dimensional  marginals is equivalent to convergence of all possible linear projections to real-valued random variable. We restate Definition 7 of \cite{matthews2018gaussian} as follows
\begin{defn}
The projections are defined in terms of a finite linear projection of the l-th layer's input values without the biases:

\begin{align}
S^{(l)}(\projectionIndeces,\projectionCoefficients)\atWidth = \sum_{(\datapoint,\widthSymbol) \in \projectionIndeces} \projectionCoefficients^{(\datapoint,\widthSymbol)}\left[h_i^l(x)\atWidth - \sigma_b \beta_i^l\right]\equationFullStop
\end{align}

\noindent where $\projectionIndeces \subset \indexSet \times \naturalNumbers$ is a finite set of tuples of data points and indices of pre-activations, with $\indexSet = (\inputSymbol[i])_{i=1}^\infty$. $\projectionCoefficients \in \realLine^{|\projectionIndeces|}$ is a vector parameterising the linear projection. The suffix $\atWidth$ indicates $\min \{n_1,\dots,n_l\}$. Let $n \rightarrow \infty$ means that the widths of 1 to l layers tend to infinity simultaneously.

We reorganize the index and let
\begin{align}\summand^{l}_{\widthSymbolB}(\projectionIndeces,\projectionCoefficients)\atWidth := \sum_{(\datapoint,\widthSymbol) \in \projectionIndeces} \projectionCoefficients^{(\datapoint,\widthSymbol)} W_{ij}^l \phi(h_i^{l-1}(x))\atWidth\frac{\sigma_w}{\sqrt{n_{l-1}}} \, ,
\end{align}
so that
\begin{align}\label{eq:projection_as_array}
S^{(l)}(\projectionIndeces,\projectionCoefficients)\atWidth = \frac{1}{\sqrt{n_{l-1}(n)}}\sum_{\widthSymbolB=1}^{n_{l-1}}\summand^{l}_{\widthSymbolB}(\projectionIndeces,\projectionCoefficients)\atWidth\, .
\end{align}
\end{defn}
Since we impose Haar ensemble to $\{W_{ij}\}_{1\leq i,j \leq n_{l-1}}$, it's invariant under the column permutation. This implies that $\{\summand^{l}_{\widthSymbolB}(\projectionIndeces,\projectionCoefficients)\atWidth\}_{1\leq j \leq n_{l-1}}$ form an exchangeable sequence with respect to the column index $j$.

To fit the three conditions of Lemma \ref{lemma:gclta}, we need the following lemma on the moment calculation of orthogonal random matrices:

\begin{lem}\label{lem:expect}
If $(W_{ij})_{n\times n}$ is an  orthogonal matrix distributed according to Haar measure, 
then $\mathbb{E}\left[\prod W_{ij}^{k_{ij}}\right]$ is non-zero
if and only if the number of entries from each row and from each
column is even.  Second and fourth-degree moments are as follows:
\begin{enumerate}
\item For all $i,j$, $$\mathbb{E}\left[W_{ij}^2\right]=\frac{1}{n}.$$ 
\item For all $i,j,r,s,\alpha,\beta,\lambda,\mu$,
\begin{equation*}\begin{split}
\mathbb{E}\big[W_{ij}W_{rs}&W_{\alpha\beta}W_{\lambda \mu}\big]\\&=
-\frac{1}{(n-1)n(n+2)}\Big[\delta_{ir}\delta_{\alpha\lambda}\delta_{j\beta}
\delta_{s\mu}+\delta_{ir}\delta_{\alpha\lambda}\delta_{j\mu}\delta_{s\beta}+
\delta_{i\alpha}\delta_{r\lambda}\delta_{js}\delta_{\beta\mu}\\&
\qquad\qquad\qquad\qquad\qquad\qquad+
\delta_{i\alpha}\delta_{r\lambda}\delta_{j\mu}\delta_{\beta s}+
\delta_{i\lambda}\delta_{r\alpha}\delta_{js}\delta_{\beta \mu}+
\delta_{i\lambda}\delta_{r\alpha}\delta_{j\beta}\delta_{s\mu}\Big]\\&\qquad
+\frac{n+1}{(n-1)n(n+2)}\Big[\delta_{ir}\delta_{\alpha\lambda}\delta_{js}
\delta_{\beta\mu}+\delta_{i\alpha}\delta_{r\lambda}\delta_{j\beta}\delta_{
s\mu}+\delta_{i\lambda}\delta_{r\alpha}\delta_{j\mu}\delta_{s\beta}
\Big].
\end{split}\end{equation*}
\label{hugemess}
\end{enumerate}
\end{lem}
\begin{rem}
In our scaling setting, the second moment should multiply by $n$ and the fourth moment should multiply by $n^2$.
\end{rem}
Let $X_{n,i} := \summand^{l}_{\widthSymbolB}(\projectionIndeces,\projectionCoefficients)\atWidth$, then
\begin{align} \nonumber
\mathbb{E}_n [X_{n,i} X_{n,j}] &= \sigma_w^2 \cdot \alpha^T \mathbb{E}_n [(W^l \phi(h_i^{l-1}(x)) ) (W^l \phi(h_i^{j-1}(x)) ) ] \\
& = \sigma_w^2 \sum_k \alpha_k^2 \mathbb{E}_n [(W^l_{ki} \phi(h_i^{l-1}(x)) ) (W^l_{kj} \phi(h_i^{j-1}(x)) ) ]
\end{align}
Note that for $i \ne j$,
$$\mathbb{E} [W_{ki}W_{kj}] = 0$$
we have
$$\mathbb{E}_n [X_{n,i} X_{n,j}] = 0.$$
Thus, condition a) of Lemma \ref{lemma:gclta} is satisfied.

The rest of the proof goes by induction. We assume that for the $l-1$-th layer, the pre-activations tend to independent Gaussian processes as the previous layers tend to infinite width simultaneously. Then we check that the moment conditions b) and c) of Lemma \ref{lemma:gclta} for the l-th layer under the inductive hypothesis and Lemma \ref{lem:expect}. We first calculate the covariance formula non-rigorously where we take the limit sequentially. Since the recursion formula of the covariance doesn't depend on how we take the limit.

We impose an assumption on the nonlinear activation function $\phi (x)$:
\begin{defn}[Lipschitz nonlinearity] \label{def2}
	A nonlinearity $\nonlinearity: \realLine \mapsto \realLine$ is said to obey the the Lipschitz property if there exist $\envelopeconstant,\envelopegradient \geq 0$ such that the following inequality holds
	
	\begin{align}
	|\nonlinearity(x)| \leq \envelopeconstant + \envelopegradient |x | \hspace{6 pt} \forall x \in \realLine \, .
	\end{align}
	
\end{defn}

\begin{prop}\label{prop:output_limit}
For a network of depth $L$ at initialization, with a Lipschitz nonlinearity $\phi$, and as $n_1, ..., n_{l-1} \to \infty$ sequentially, the pre-activations $h^l_{i}$, for $i=1, ..., n_{l}$, tend to i.i.d centered Gaussian processes specified by the covariance $\{\Sigma^{l}\}_{1 \leq l \leq L}$, where $\Sigma^{l}$ is defined recursively by:
\begin{align*}
\Sigma^{1}(x, x') &= \frac{\sigma^2_w}{n_0} x^T x' + \sigma^2_b \\
\Sigma^{l}(x, x') &= \sigma^2_w \mathbb{E}_{f\sim\mathcal{N}\left(0,\Sigma^{l-1}\right)}[\phi(f(x)) \phi(f(x'))] + \sigma^2_b,
\end{align*}
taking the expectation with respect to a centered Gaussian process of covariance $\Sigma^{l-1}$ denoted by $f$. In the CNN case, $\Sigma_{\alpha,\alpha'}^{l}$ is defined recursively by:
\begin{align*}
\Sigma^{1}_{\alpha,\alpha'}(x, x') &= \frac{\sigma^2_w}{n_0(2k+1)} \sum_{\beta=-k}^k  x_{\alpha+\beta}^T x'_{\alpha'+\beta} + \sigma^2_b \\
\Sigma^{l}_{\alpha,\alpha'}(x, x') &= \frac{\sigma^2_w}{(2k+1)} \sum_\beta \mathbb{E}_{f\sim\mathcal{N}\left(0,\Sigma_{\alpha+\beta,\alpha'+\beta}^{l-1}\right)} [\phi(f(x_{\alpha+\beta})) \phi(f(x_{\alpha'+\beta}'))] + \sigma^2_b,
\end{align*}
where $\alpha$ is the convolution index.
\end{prop}

\begin{proof} We show the result in the framework of NTK parameterization, while the argument for standard parameterization can be derived in the same way. In the FCN case, When $L=1$, there are no hidden layers and $h^1_i$ has the form: 
$$
 h^1_i = \sum_{j=1}^{n_0} \frac{\sigma_w}{\sqrt{n_0}} W_{ij}^{1} x^0_j + \sigma_b \beta_i^{1}.
$$
Then we check the variance $\Sigma^{1}$ of output layer $h^l_i$. By Lemma \ref{lem:expect}, 
$$
 \mathbb{E} \big[W^1_{ij} \big] = 0,~~~ \mathbb{E} \big[W_{ij}^1 W^1_{kl}\big] = \frac{1}{n_0} [n_0 \delta_{ik}\delta_{jl}] = \delta_{ik}\delta_{jl}.
$$ 
We note that the fraction $\frac{1}{n_0}$ is from the scaling of the orthogonal distribution, while the term $n_0$ comes from the initialization. Thus we can compute the covariance of the first layer explicitly:
\begin{align*}
  \Sigma^1 = & \mathbb{E} \big[ h^1_i h^1_i \big] = \mathbb{E} \big[ (\sum_{j=1}^{n_0} \frac{\sigma_w}{\sqrt{n_0}} W_{ij}^{1} x^0_j + \sigma_b \beta_i^{1})(\sum_{j'=1}^{n_0} \frac{\sigma_w}{\sqrt{n_0}} W_{ij'}^{1} x^0_{j'} + \sigma_b \beta_i^{1}) \big] \\
 = & \frac{\sigma^2_w}{n_0} \mathbb{E}\big[ \sum_{j=1}^{n_0} \sum_{j'=1}^{n_0} W^1_{ij} x^0_j W^1_{ij'} x^0_{j'} \big] + \sigma^2_b  = \frac{\sigma^2_w}{n_0} x^T x + \sigma^2_b.
\end{align*} 

The next step is to use the induction method. Consider an $l$-network as the function mapping the input to the pre-activations $h^l_i$. The induction hypothesis gives us that as $n_1, ..., n_{l-2} \to \infty$ sequentially, the pre-activations $h^{l-1}_i$ tend to i.i.d Gaussian processes with covariance $\Sigma^{l-1}$. Then the inputs of the l-th layer are governed by:
$$
h^{l}_{i} = \frac{\sigma_w}{\sqrt{n_{l-1}}} \sum_{j=1}^{n_{l-1}} W^l_{ij}  x_j^{l-1} + \sigma_b \beta_i^{l},
$$
where $x_j^{l-1} : = \phi (h^{l-1}_{j})$.
When the input $x =x'$, let $$B = \sqrt{\frac{n_{l-1}}{\sum_{i=1}^{n_{l-1}} (x^{l-1}_i)^2}} \left[ 
\begin{matrix}
x^{l-1}_1 & 0 &\cdots & 0\\
\vdots & \vdots &\ddots & \vdots \\
x_{n_{l-1}}^{l-1} & 0 &\cdots & 0 \\
\end{matrix}
\right ],$$ then $tr(BB^T) = n_{l-1}$. By Lemma \ref{lem:haar}, we have that $lim_{n \rightarrow \infty} \sum^{n_{l-1}}_{j=1} \frac{1}{\sqrt{n_{l-1}}} W^l_{ij} x^{l-1}_j$ tends to $\mathcal{N} (0, lim_{n \rightarrow \infty} \frac{\sum_{i=1}^{n_{l-1}} (x^{l-1}_i)^2}{n_{l-1}} )$. 
As a result, $\{h_i^{l}\}$ are centered Gaussian variables. With the help of Lemma \ref{lem:expect}, we get the following covariance expression between two input data $x$ and $x'$:
$$\Sigma^{l}(x,x')[n_{l-1}] = \frac{\sigma^2_w}{n_{l-1}} \displaystyle\sum_{j} \phi(h^{l-1}_j (x))\phi(h^{l-1}_j (x')) + \sigma^2_b.$$
Since $h^{l-1}_j$, $h^{l-1}_k$ are independent for $j \neq k$ by the inductive hypothesis, we have 
$$Var(\Sigma^{l}(x,x')) = \frac{\sigma^4_w}{n_{l-1}^2} \displaystyle\sum_{j} \mathbb{E} \big[ \phi(h^{l-1}_j (x))^2 \phi(h^{l-1}_j (x'))^2 \big] - \frac{\sigma^4_w}{n_{l-1}^2}\displaystyle\sum_{j}\mathbb{E} \big[ \phi(h^{l-1}_j (x)) \phi(h^{l-1}_j (x')) \big] ^2 .$$
By the symmetry of the underlying index j, we can choose index $1$ as a representative:
$$ Var(\Sigma^{l}(x,x'))= \frac{\sigma^4_w}{n_{l-1}} \mathbb{E} \big[ \phi(h^{l-1}_1 (x))^2 \phi(h^{l-1}_1 (x'))^2 \big] - \frac{\sigma^4_w}{n_{l-1}} \mathbb{E} \big[ \phi(h^{l-1}_1 (x)) \phi(h^{l-1}_1 (x')) \big] ^2.$$
Since $\mathbb{E} \big[ \phi(h^{l-1}_1 (x))^2 \phi(h^{l-1}_1 (x'))^2 \big] -  \mathbb{E} \big[ \phi(h^{l-1}_1 (x)) \phi(h^{l-1}_1 (x')) \big]$ is bounded, by the Chebyshev's inequality,
the covariance kernel tends in probability to its expectation,
$$\Sigma^{l}(x, x') := \lim_{n_{l-1} \rightarrow \infty}\Sigma^{l}(x,x')[n_{l-1}] = \sigma^2_w \mathbb{E}_{f\sim\mathcal{N}\left(0,\Sigma^{l-1}\right)}[\phi(f(x)) \phi(f(x'))] + \sigma^2_b.$$

We still need to verify the independence of $h_i^l$, $h_j^l$  for $i \neq j$. This follows from computing the covariance between $h_i^l$, $h_j^l$:
$$\lim_{n_{l-1} \rightarrow \infty} Cov(h_i^{l} (x) h_j^{l} (x')) = \displaystyle\sum_{k}\displaystyle\sum_{l} W_{ik}^{l} \phi(h^{l-1}_k (x)) W_{jl}^{l} \phi(h^{l-1}_l (x'))$$
Note that we ignore the bias $b_i^l$, since they are independent with the weight parameters $\{W^l_{ij}\}$. By Lemma \ref{lem:expect}, 
$$\mathbb{E} \big[W_{ik}W_{jl} \big] = 0, for ~~i \neq j.$$
Therefore we have, 
$$ \lim_{n_{l-1} \rightarrow \infty} Cov(h_i^l(x) h_j^l (x')) = 0.$$
Since we know that as $n_{l-1} \rightarrow \infty$, $h_i^l (x)$, $h_j^l (x')$ are Gaussian variables, zero covariance means they are independent.

In the CNN case, we have an extra convolution index $\alpha$. Since $\mathbb{E}[(W^l_{\alpha})(W^l_{\alpha'})] = 0$ for $\alpha \neq \alpha'$, we have
$$\Sigma^{l}_{\alpha,\alpha'}(x, x') = \frac{\sigma^2_w}{(2k+1)} \sum_\beta\mathbb{E}_{f\sim\mathcal{N}\left(0,\Sigma_{\alpha+\beta,\alpha'+\beta}^{l-1}\right)} [\phi(f(x_{\alpha+\beta})) \phi(f(x_{\alpha'+\beta}'))] + \sigma^2_b.$$

\end{proof}
With the explicit covariance formula in hand, we can prove by induction that the limiting covariance of the finite width output functions match with the covariance of the infinite Gaussian processes obtained above.
\begin{prop}
Let $\sigma^2 (l,\projectionIndeces,\projectionCoefficients)\atWidth$ be the variance of the random variable $\summand^{(l)}_{\widthSymbolB}(\projectionIndeces,\projectionCoefficients)\atWidth$, then

\begin{align}
\lim_{\sequenceVariable \to \infty}\sigma^2 (l,\projectionIndeces,\projectionCoefficients)\atWidth = \sigma^2 (l,\projectionIndeces,\projectionCoefficients)[\infty] \, ,
\end{align}

where 
\begin{align}
\sigma^2 (l,\projectionIndeces,\projectionCoefficients)[\infty] = \alpha^{T} [\Sigma^l (x,x') - \sigma_b^2] \alpha .
\end{align}
\end{prop}
\begin{proof}
$$\sigma^2 (l,\projectionIndeces,\projectionCoefficients)\atWidth = \frac{\sigma_w^2}{n_{l-1}} \alpha^T\mathbb{E}[W^l\phi(h^{l-1}_1) [W^l\phi(h^{l-1}_1)^T]\alpha.$$
By Lemma \ref{lem:expect} and independence between l-th layer and (l-1)-th layer, we have
$$\mathbb{E}[W^l_{i1}\phi(h^{l-1}_1) [W^l_{j1}\phi(h^{l-1}_1)^T] = \frac{\sigma_w^2}{n_{l-1}}\delta(i=j) \mathbb{E}
[\phi(h^{l-1}_1(x_i)[n])\phi(h^{l-1}_1(x_j)[n])].$$
Once we prove that $\phi(h^{l-1}_1(x_i)[n])\phi(h^{l-1}_1(x_j)[n])$ is uniformly integrable with respect to $n$, the proof goes exactly the same as the proof of Lemma 17 in \cite{matthews2018gaussian}. To build the uniform integrability, we need the Lipschitz nonlinearity property of $\phi$.
$$\mathbb{E}
[\phi(h^{l-1}_1[n]_i)\phi(h^{l-1}_1[n]_j)] \leq \mathbb{E} [(c + m |h^{l-1}_1(x_i)[n])|\cdot (c + m |h^{l-1}_1(x_j)[n])| ].$$
To claim we have a uniform bound for the right hand side, it suffices to show that $\mathbb{E} [|h^{l-1}_1(x_i)[n]|]$ and $\mathbb{E} [|h^{l-1}_1(x_i)[n]h^{l-1}_1(x_j)[n]|]$ are uniformly bounded. By the boundedness of Pearson correlation coefficient, we are down if
$$\mathbb{E} [|(h^{l-1}_1(x)[n])|^2]$$
is uniformly bounded as $n \rightarrow \infty$. This can be down by induction.

Since $\{W_{ij}\}$ is a scaled orthogonal matrix,
\begin{align*}
n_{l} \mathbb{E} [|h^{l}_1(x)[n]|^2 &= \mathbb{E} [\norm{h_{1:n_l}^l (x)[n]}^2]\\
& = \mathbb{E} [\mathbb{E} [\norm{h_{1:n_l}^l (x)[n]}^2 | h^{l-1}_{1:n_{l-1}}(x)[n]]\\
& = \sigma_w^2 \mathbb{E}[\norm{\phi(h_{1:n_{l-1}}^{l-1} (x)[n])}^2]
\end{align*}
Applying the lipschitz nonlinearity property, we get
\begin{align}
\mathbb{E}[\norm{\phi(h_{1:n_{l-1}}^{l-1} (x)[n])}^2] &= \mathbb{E} [\sum_{j=1}^{j=n} \phi^2(h_j^{l-1} (x)[n])]\\
& \leq \mathbb{E} [\sum_{j=1}^{j=n} (c+ m |h_j^{l-1} (x)[n]|)^2 ].
\end{align}
Each term in the expansion of the square is bounded by $\mathbb{E} [|h_j^{l-1} (x)[n]|^2]$, which is uniformly bounded by the inductive hypothesis. The number of terms in the expansion is of order $n$, thus $\mathbb{E} [|h^{l}_1(x)[n]|^2$ is uniformly bounded. Therefore, we can safely change the order of limit as in the proof of Lemma 17 in \cite{matthews2018gaussian}.
\end{proof}
\begin{rem}
Another way to bound the $\mathbb{E} [|h^{l}_1(x)[n]|^2$ is to apply Lemma \ref{lem:expect} to the recursion relation directly:
$$\mathbb{E} [|h^{l}_1(x)[n]|^2 ]= \frac{\sigma_w^2}{n} \mathbb{E} [\sum_j W_{1j} \phi(h^{l-1}_j(x)[n])]^2.$$
\end{rem}
To verify condition (2) and (3) of Lemma \ref{lemma:gclta} under the inductive hypothesis, we go through the same procedure 
as the proof of lemmas 15 and 16 in \cite{matthews2018gaussian}. We neglect the detail and conclude that by Lemma \ref{lem:expect} and the inductive hypothesis,
\begin{align}
S^{(l)}(\projectionIndeces,\projectionCoefficients)\atWidth &= \frac{1}{\sqrt{n_{l-1}(n)}}\sum_{\widthSymbolB=1}^{n_{l-1}}\summand^{l}_{\widthSymbolB}(\projectionIndeces,\projectionCoefficients)\atWidth\, \\
& \stackrel{P}{\longrightarrow} \normal(0,\sigma^2[\infty]).
\end{align}
Since it holds for arbitary linear functional $\alpha$, we have shown that $\{h_i^l (x)\}$ converge to independent Gaussian processes, where the covariance is specified by the recursion formula in Proposition 1. In the CNN case, adding the convolution index $\alpha$ won't change the symmetry of index $j$ and Lemma \ref{lemma:gclta} can be extended to CNN in the same way without any nontrivial change. Therefore we have simultaneously proved the following proposition for FNN and CNN:

\begin{thm}\label{thm:output_limit}
Consider a FCN at orthogonal initialization, with a Lipschitz nonlinearity $\phi$, and in the limit as $n_1, ..., n_{L-1} \to \infty$, the pre-activations $h^l_{i}$, for $i=1, ..., n_{l}$ for $l \in \{1,\dots, L\}$, tend to i.i.d centered Gaussian processes of covariance $\Sigma^{l}$ which is defined recursively by:
\begin{align*}
\Sigma^{1}(x, x') &= \frac{\sigma^2_w}{n_0} x^T x' + \sigma^2_b \\
\Sigma^{l}(x, x') &= \sigma^2_w \mathbb{E}_{f\sim\mathcal{N}\left(0,\Sigma^{l-1}\right)}[\phi(f(x)) \phi(f(x'))] + \sigma^2_b,
\end{align*}

For a CNN at orthogonal initialization, and in the limit as $n_1, ..., n_{L-1} \to \infty$, the pre-activations $h^l_{i,\alpha}$ tend to Gaussian processes of covariance $\Sigma_{\alpha,\alpha'}^{l}$ for $l \in \{1,\dots, L-1\}$, which is defined recursively by:
	\begin{align*}
	\Sigma^{1}_{\alpha,\alpha'}(x, x') &= \frac{\sigma^2_w}{n_0(2k+1)} \sum_{\beta=-k}^k  x_{\alpha+\beta}^T x'_{\alpha'+\beta} + \sigma^2_b \\
	\Sigma^{l}_{\alpha,\alpha'}(x, x') &= \frac{\sigma^2_w}{(2k+1)} \sum_{\beta=-k}^k \big[ \mathbb{E}_{f\sim\mathcal{N}\left(0,\Sigma_{\alpha+\beta,\alpha'+\beta}^{l-1}\right)}\\
	&        [\phi(f(x_{\alpha+\beta})) \phi(f(x_{\alpha'+\beta}'))] \big] + \sigma^2_b.
	\end{align*}
with covariance of output layer $\Sigma^{L}(x, x') =  {\rm Tr} \big[ \mathbb{E}_{f\sim\mathcal{N}\left(0,\Sigma_{\alpha,\alpha'}^{L-1}\right)}  [\phi(f(x_{\alpha})) \phi(f(x_{\alpha'}'))] \big]$ .
\end{thm}

The covariance of last layer is a trace over $\alpha,\alpha' \in m \times m$, because the last layer is fully connected:
$$
\Sigma^{L}(x, x') =  \sum_\alpha \delta_{\alpha,\alpha'} \big[ \mathbb{E}_{f\sim\mathcal{N}\left(0,\Sigma_{\alpha,\alpha'}^{L-1}\right)}  [\phi(f(x_{\alpha})) \phi(f(x_{\alpha'}'))] \big]
$$

\subsection{NTK at Initialization}

In the infinite-width limit, the neural tangent kernel which is random at networks with orthogonal initialization, converges in probability to a deterministic limit. 
\begin{thm}\label{thm:convergence_NTK_initialization}
Consider a FCN at orthogonal initialization, with a Lipschitz nonlinearity $\phi$, and in the limit as the layers width $n_1, ..., n_{L-1} \to \infty$, the NTK $\Theta^L_0(x,x')$, converges in probability to a deterministic limiting kernel: $$\Theta_0^{L}(x,x') \to \Theta^{L}_\infty(x,x') \otimes {\bf I}_{n_L \times n_L}.$$
The scalar kernel $\Theta^{L}_\infty(x,x')$ is defined recursively by
\begin{align*}
    \Theta^{1}_\infty(x, x') &= \Sigma^{1}(x, x') \\
    \Theta^{l}_\infty(x, x') &= \sigma^2_w  \dot{\Sigma}^{l}(x, x') \Theta^{l-1}_\infty(x, x')  + \Sigma^{l}(x, x'),
\end{align*}
where
\[
	\dot{\Sigma}^{l}\left(x,x'\right)=
	\mathbb{E}_{f\sim\mathcal{N}\left(0,\Sigma^{\left(l-1\right)}\right)}\left[\dot{\phi}\left(f\left(x\right)\right)\dot{\phi}\left(f\left(x'\right)\right)\right],
\]

For a CNN at orthogonal initialization, and in the infinite-channel limit, the NTK $\Theta^{L}_0(x,x')$, converges in probability to a deterministic limiting kernel: $$\Theta^{L}_0(x,x') \to \Theta^{L}_\infty(x,x') \otimes {\bf I}_{n_L \times n_L}.$$
The scalar kernel $\Theta^{L}_\infty(x,x')$ is given recursively by
\begin{align*}
   & {\Theta_{\alpha,\alpha'}^{1}}_\infty(x, x') = \Sigma_{\alpha,\alpha'}^{1}(x, x') \\
   & {\Theta_{\alpha,\alpha'}^{l}}_\infty(x, x') = \frac{\sigma^2_w}{(2k+1)} \sum_\beta \big[ \dot{\Sigma}_{\alpha+\beta,\alpha'+\beta}^{l}(x, x')
                                              {\Theta_{\alpha+\beta,\alpha'+\beta}^{l-1}}_\infty(x, x')  + \Sigma_{\alpha+\beta,\alpha'+\beta}^{l}(x, x') \big] \\
   & \Theta^L_\infty(x,x')  = {\rm Tr}\big[  \dot{\Sigma}_{\alpha,\alpha'}^{L}(x, x') 
                                                 {\Theta_{\alpha,\alpha'}^{L-1}}_\infty(x, x')  + \Sigma_{\alpha,\alpha'}^{L}(x, x')   \big]
\end{align*}
\end{thm}


\begin{proof}
We show the result in the framework of NTK parameterization, while the argument for standard parameterization can be derived in the same way with a similar result, see the details for the Gaussian initialization in \cite{sohl2020infinite}. For the input layer $L = 1$, taking the derivative with respect to $W^1_{ij}$, $b^1_j$, we have
\begin{align*}
\Theta^{1}_{kk'} (x,x') &= \frac{\sigma^2_w}{n_0} \displaystyle\sum_{i = 1}^{ n_0} \displaystyle\sum_{j = 1}^{n_0} x_i x_i ' \delta_{jk}\delta_{jk'} + \sigma^2_b \displaystyle\sum_{j = 1}^{n_0} \delta_{jk}\delta_{jk'}\\
& =\frac{\sigma^2_w}{n_0}x^T x' \delta_{kk'} + \sigma^2_b \delta_{kk'}.
\end{align*}
From $(l-1)$-th layer to $l$-th layer, by the inductive hypothesis, as $n_1, \dots, n_{l-2} \rightarrow \infty$, the pre-activations $h^{l-1}_i$ are i.i.d centered Gaussian with covariance $\Sigma^{l-1}$ and $\Theta^{l-1}_{ii'} (x,x')$ converges to:
$$(\partial_{\theta} h^{l-1}_i(x))^T (\partial_{\theta} h^{l-1}_{i'}(x')) \rightarrow \Theta^{l-1}_{\infty} (x,x') \delta_{ii'}.$$
Now we calculate the NTK for $l$ layer network and divide the parameters into two parts. The first part only involves the parameters of the $l$-th layer, and the other is the collection of parameters from previous $1,\dots ,l-1$ layers. 
The first part of the NTK is given by
$$\alpha_l := \partial_{w^l} h^l_k (x) \cdot \partial_{w^l} h^l_k (x') + \partial_{b^l} h^l_k (x) \cdot \partial_{b^l} h^l_k (x') = \sum_i \frac{\sigma^2_w}{n_{l-1}}\phi(h^{l-1}_i (x))\phi(h^{l-1}_i (x')) + \sigma^2_b. $$
Note that in the Gaussian initialization,
$$\phi(h^{l-1}_i (x))\phi(h^{l-1}_i (x'))$$
is independent of
$$\phi(h^{l-1}_j (x))\phi(h^{l-1}_j (x')),$$
for $i \ne j$ before taking the $n_{l-1} \rightarrow \infty$ limit. Thus for the $\frac{1}{n_{l-1}}$ scaling, we already observe that $\sum_i \frac{\sigma^2_w}{n_{l-1}}\phi(h^{l-1}_i (x))\phi(h^{l-1}_i (x'))$ tends to its mean by the classical law of large numbers. If we take the limit sequentially, since $h_i^{l-1}(x)$ and $h_j^{l-1}(x)$ are independent Gaussian by the previous section, we know that it tends to $\Sigma^l (x,x')$ for the same reason.

If we want to know how $\alpha_l$ behaves when the width tends to infinity  simultaneously, the nonlinear activation would break the non-asymptotic independence. So we first assume that we work in the linear network category. Then
\begin{align*}
Var[\alpha_l [n]] &= (\frac{\sigma_w^2}{n_{l-1}[n]})^2 Var[\sum_i h_i^{l-1}(x)\cdot h_i^{l-1}(x')]\\
& = (\frac{\sigma_w^2}{n_{l-1}[n]})^2 (\sum_{i=1}^n Var [h_i^{l-1}(x)h_i^{l-1}(x')] +
\sum_{i \ne j} Cov[h_i^{l-1}(x)h_i^{l-1}(x'),h_j^{l-1}(x)h_j^{l-1}(x')]),
\end{align*}
where
\begin{align*}Cov[h_i^{l-1}(x)h_i^{l-1}(x')h_j^{l-1}(x)h_j^{l-1}(x')] = & \mathbb{E}[h_i^{l-1}(x)h_i^{l-1}(x'),h_j^{l-1}(x)h_j^{l-1}(x')] \\ &- \mathbb{E}[h_i^{l-1}(x)h_i^{l-1}(x')]\mathbb{E}[h_j^{l-1}(x)h_j^{l-1}(x')].\end{align*}
By Lemma \ref{lem:expect}, 
$$\mathbb{E} [h_i^{l-1}h_i^{l-1}] = \frac{\sigma_w^2}{n_{l-2}} \sum_{k=1}^{n_{l-2}} \mathbb{E}[\phi^2(h_k^{l-2})].$$
The right-hand side is independent of index $i$ as expected.
\begin{align*}
\mathbb{E}[h_i^{l-1}(x)h_i^{l-1}(x'),h_j^{l-1}(x)h_j^{l-1}(x')] & = (\frac{\sigma_w^2}{n_{l-2}[n]})^2 \frac{n_{l-2}^2(n_{l-2}+1)}{(n-1)n(n+2)} (\sum_{k=1}^{n_{l-2}} \mathbb{E}[\phi^2(h_k^{l-2})])^2 \\
&+ (\frac{\sigma_w^2}{n_{l-2}[n]})^2 \frac{n_{l-2}^3(n_{l-2}-1)}{(n-1)n(n+2)}\mathbb{E}[\phi^2(h_1^{l-2})\phi^2(h_2^{l-2})]
\end{align*}
Thus
$$Cov[h_i^{l-1}(x)h_i^{l-1}(x')h_j^{l-1}(x)h_j^{l-1}(x')] \sim O(\frac{1}{n_{l-2}[n]}).$$
This implies that
$$Var[\alpha_l [n]] \sim O(\frac{1}{n}).$$
Let $\mu_n$ denote the mean of $\alpha_l [n]$, then by Chebyshev's inequality, for $\forall \epsilon > 0$,
$$P(|\alpha_l [n] - \mu_n|\ge \epsilon) \sim O(\frac{1}{n}).$$
Since $\lim_{n \rightarrow \infty} \mu_n = \Sigma_l$, we have
$$P(|\alpha_l [n] - \Sigma_l|\ge \epsilon) \sim O(\frac{1}{n}),$$
for n large enough. In conclusion, we have $\alpha_l [n]$ tends to $\Sigma_l(x,x')$. \\

In the non-linear activation case, if the activation satisfies definition \ref{def2}, we still have the asymptotic result
$$Var[\alpha_l [n]] \rightarrow 0.$$ Therefore, by Chebyshev's inequality, we get
$$P(|\alpha_l [n] - \Sigma_l|\ge \epsilon) \rightarrow 0,$$
for $\forall$ $\epsilon > 0$.

For the second part, denoting the parameters of the previous $l- 1$ layers as $\tilde{\theta}$, we have
$$\partial_{\tilde{\theta}} h^l_k (x) = \frac{\sigma_w}{\sqrt{n_{l-1}}}\displaystyle\sum_{i = 1}^{n_{l-1}} \partial_{\tilde{\theta}}h_i^{l-1} (x) \dot{\phi}(h^{l-1}_i (x)) W_{ik}^l.$$
By the induction hypothesis, the NTK of $(l-1)$-layer networks $\Theta^{l-1}_{kk'} (x,x')$ converges to a diagonal kernel as $n_1, \dots, n_{l-2} \rightarrow \infty$. we denote the second part of NTK by,
$$\alpha_{l-1} [n] := \frac{\sigma^2_w}{n_{l-1}} \displaystyle\sum_{i,i'= 1}^{n_{l-1}} \Theta^{l-1}_{ii'} (x,x')\dot{\phi}(h^{l-1}_i (x)) \dot{\phi}(h^{l-1}_{i'} (x'))W_{ik}^l W_{i'k'}^l.$$
Note that $\{W_{ik}^l W_{i'k'}^l\}$ is independent of $\Theta^{l-1}_{ii'} (x,x')\dot{\phi}(h^{l-1}_i (x)) \dot{\phi}(h^{l-1}_{i'} (x'))$, this allows us to prove that it converges to a deterministic kernel by induction.\\
By Lemma \ref{lem:expect},
$$\mathbb{E}[\alpha_{l-1} [n]] = \frac{\sigma^2_w}{n_{l-1}} \delta_{kk'}\delta_{ii'} \displaystyle\sum_{i,i'= 1}^{n_{l-1}}\mathbb{E}[\Theta^{l-1}_{ii'} (x,x')\dot{\phi}(h^{l-1}_i (x)) \dot{\phi}(h^{l-1}_{i'} (x'))],$$
$$Var[\alpha_{l-1} [n]] = \mathbb{E}\big[  ( \alpha_{l-1} - \mathbb{E}[\alpha_{l-1}])^2 \big] = \frac{\sigma^4_w}{n^2_{l-1}}\mathbb{E} \big[ \big( \displaystyle\sum_{i,i'= 1}^{n_{l-1}} \Theta^{l-1}_{ii'} (x,x')\dot{\phi}(h^{l-1}_i (x)) \dot{\phi}(h^{l-1}_{i'} (x'))(W_{ik}^l W_{i'k'}^l -  \delta_{kk'}\delta_{ii'}) \big)^2\big].$$
Expanding the square and note that
$$\mathbb{E}[W_{ik}W_{i'k'}W_{jk}W_{j'k'}] = 0,$$
if $i \ne j \ne i' \ne j'$. This implies that $n_l (n_{l-1}-1)(n_{l-1}-2)(n_{l-1}-3)$ number of terms in the expansion are zero. Now we reorganize the left terms into three groups:
\begin{itemize}
 \item $i \ne i'$:\\
 The only terms in the expansion that survive after taking the expectation are of the form:
 $W_{ik}W_{i'k'}W_{ik'}W_{i'k}.$ The expectation is separated into
 $$\mathbb{E}[\Theta^{l-1}_{ii'} (x,x')\dot{\phi}(h^{l-1}_i (x)) \dot{\phi}(h^{l-1}_{i'} (x'))]^2 \cdot \mathbb{E}[W_{ik}W_{i'k'}W_{ik'}W_{i'k}].$$
 There are $O(n_{l-1}\cdot (n_{l-1}-1))$ such terms in the expansion, by Lemma \ref{lem:expect},
 $$\mathbb{E}[W_{ik}W_{i'k'}W_{ik'}W_{i'k}] \sim O(1).$$
 We conclude that
 $$\frac{\sigma^4_w}{n^2_{l-1}}\mathbb{E} \big[ \big( \displaystyle\sum_{i\ne i'} \Theta^{l-1}_{ii'} (x,x')\dot{\phi}(h^{l-1}_i (x)) \dot{\phi}(h^{l-1}_{i'} (x'))(W_{ik}^l W_{ik'}^l) \big)^2\big] \sim O(\mathbb{E}[\Theta^{l-1}_{ii'} (x,x')\dot{\phi}(h^{l-1}_i (x)) \dot{\phi}(h^{l-1}_{i'} (x'))]^2).$$
 \item $(i =i')\ne (j = j')$:\\
 We need to estimate a fourth order moment
 $$\gamma := \mathbb{E} \big[(W^l_{ik}W^l_{ik'} -  \delta_{kk'})(W^l_{i'k}W^l_{i'k'} -  \delta_{kk'})\big].$$
 If $k = k'$, by Lemma \ref{lem:expect},
 \begin{align*}
 \gamma &= -\frac{2 n^2_{l-1}}{(n_{l-1}-1)n_{l-1}(n_{l-1}+2)} + \frac{(n_{l-1}+1)n^2_{l-1}}{(n_{l-1}-1)n_{l-1}(n_{l-1}+2)} - 1 \sim O(\frac{1}{n_{l-1}}).\\
 \end{align*}
 
 If $k \neq k'$,
 $$ \mathbb{E} \big[ (W_{ik}W_{ik'}W_{i'k}W_{i'k'}) \big] = -\frac{n^2_{l-1}}{(n_{l-1}-1)n_{l-1}(n_{l-1}+2)} \sim O(\frac{1}{n_{l-1}}). $$
 In each case, we get that $\gamma \sim O(\frac{1}{n_{l-1}})$. Thus we have \\
 $$\frac{\sigma^4_w}{n^2_{l-1}} \mathbb{E} \big[  \displaystyle\sum_{i=1}^{n_{l-1}} \displaystyle\sum_{j=1}^{n_{l-1}} \big( \Theta^{l-1}_{ii}\Theta^{l-1}_{jj} (x,x')\dot{\phi}^2(h^{l-1}_i (x)) \dot{\phi}^2(h^{l-1}_{j} (x')) \big) \big( W^l_{ik}W^l_{ik'} -  \delta_{kk'} \big) \big( W^l_{jk}W^l_{jk'} -  \delta_{kk'} \big) \big] \sim O(\frac{1}{n_{l-1}}).$$
 \item $(i =i')= (j = j')$:
We have
 $$\mathbb{E} \big[ (W^l_{ik}W^l_{ik'} -  \delta_{kk'})^2 \big] = \mathbb{E} [W^l_{ik}W^l_{ik'}]^2 -  \delta_{kk'}.$$
 Note that the weights are drawn from the NTK parameterization, and by Lemma \ref{lem:expect},
 $$\mathbb{E} [ (W^l_{ik})^2 (W^l_{ik'})^2] \sim O(1).$$
 Therefore the number of terms with the same index i are of order $ O(\frac{1}{n_{l-1}})$ and we get
 $$\frac{\sigma^4_w}{n^2_{l-1}} \mathbb{E} \big[  \displaystyle\sum_{i=1}^{n_{l-1}} \big( \Theta^{l-1}_{ii} (x,x')\dot{\phi}(h^{l-1}_i (x)) \dot{\phi}(h^{l-1}_{i} (x')) \big)^2 \big( W^l_{ik}W^l_{ik'} -  \delta_{kk'} \big)^2 \big] \sim O(\frac{1}{n_{l-1}}).$$
\end{itemize}
Combing the three groups, we have
$$Var[\alpha_{l-1} [n]] \sim O(\frac{1}{n_{l-1}}) + O(\mathbb{E}[\Theta^{l-1}_{ii'} (x,x')\dot{\phi}(h^{l-1}_i (x)) \dot{\phi}(h^{l-1}_{i'} (x'))]^2).$$
By the inductive hypothesis,
$$\mu := \lim_{n_{l-1} \rightarrow \infty} \mathbb{E} [\alpha_{l-1} [n]] = \sigma_w^2 \Theta^{l-1}_{\infty} (x,x') \dot{\Sigma}^{l}\left(x,x'\right) \delta_{kk'}.$$
By Chebyshev's inequality,
$$P(|\alpha_{l-1} [n] - \mu| \ge \epsilon) \leq \frac{2}{\epsilon^2}(O(\frac{1}{n_{l-1}}) + O(\mathbb{E}[\Theta^{l-1}_{ii'} (x,x')\dot{\phi}(h^{l-1}_i (x)) \dot{\phi}(h^{l-1}_{i'} (x'))]^2)).$$
By the inductive hypothesis and the integrability result from the last section,
$$\mathbb{E}[\Theta^{l-1}_{ii'} (x,x')\dot{\phi}(h^{l-1}_i (x)) \dot{\phi}(h^{l-1}_{i'} (x'))]^2 \stackrel{n_{l-1} \rightarrow \infty}{\longrightarrow} 0.$$
Thus we obtain that second part of the NTK tends to $\sigma^2_w \Theta^{l-1}_{\infty} (x,x') \dot{\Sigma}^{l} (x,x')$ if we let the width go to infinity simultaneously.\\
Combining the two parts, we have
$$\Theta^{l}_{\infty} (x,x') =\sigma^2_w \Theta^{l-1}_{\infty} (x,x') \dot{\Sigma}^{l} (x,x') + \Sigma^{l} (x,x').$$

In the CNN case, when $L=1$,
$$\Theta^{l}_{\alpha,\alpha'}(x,x') = \frac{\sigma_w^2}{n_0 (2k+1)}\sum_{\beta}  x^T_{\alpha+\beta}x_{\alpha'+\beta} +  \sigma_b^2.$$
Here we omit the subscript $\infty$ for simplicity. Assume the NTK formula is true for $L=l-1$, then the first part of the l-layer NTK is given by
$$\alpha_L \rightarrow (\frac{\sigma_w^2}{n_0 (2k+1)}\sum_{\beta} \mathbb{E}[\phi(h_{\alpha + \beta}^{l-1}(x))\phi(h_{\alpha' + \beta}^{l-1}(x'))] + \sigma_b^2),$$
by the same moment argument as before. For the second part, we still have
$$\Theta^{l-1}_{(i,\alpha+\beta),(i',\alpha'+\beta)}(x,x') = \delta_{ii'} \Theta^{l-1}_{\alpha+\beta,\alpha'+\beta}(x,x').$$
Thus, for the second part,
$$\partial_{\tilde{\theta}}h_{k,\alpha}^l (x) \cdot \partial_{\tilde{\theta}}h_{k,\alpha'}^l (x) \rightarrow \sigma_w^2 \Theta^{l-1}_{\alpha+\beta,\alpha'+\beta}(x,x')\dot{\Sigma}_{\alpha+\beta,\alpha'+\beta}^{l}(x, x'),$$
since $\mathbb{E} [W_{\alpha}W_{\alpha'} ] = 0$ for $\alpha \ne \alpha'$. We conclude that
$${\Theta_{\alpha,\alpha'}^{l}}_\infty(x, x') = \frac{\sigma^2_w}{(2k+1)} \sum_\beta \dot{\Sigma}_{\alpha+\beta,\alpha'+\beta}^{l}(x, x') {\Theta_{\alpha+\beta,\alpha'+\beta}^{l-1}}_\infty(x, x')  + \Sigma_{\alpha+\beta,\alpha'+\beta}^{l}(x, x').$$

For the last (fully-connected) layer, we have,
$$
\begin{aligned}
\Theta^L_\infty(x,x') & = \sum_\alpha \delta_{\alpha,\alpha'} \big[  \dot{\Sigma}_{\alpha,\alpha'}^{L}(x, x')  {\Theta_{\alpha,\alpha'}^{L-1}}_\infty(x, x')  + \Sigma_{\alpha,\alpha'}^{L}(x, x')   \big] \\
                      & = {\rm Tr}\big[  \dot{\Sigma}_{\alpha,\alpha'}^{L}(x, x')  {\Theta_{\alpha,\alpha'}^{L-1}}_\infty(x, x')  + \Sigma_{\alpha,\alpha'}^{L}(x, x')   \big] 
\end{aligned}
$$
 
\end{proof}
\begin{rem}
Note that the proof for the NNGP and NTK behavior is mainly based on the moment calculation in Lemma \ref{lem:expect}. Let's take the second moment as an example. $\mathbb{E}[W_{ij}W_{rs}]$can be seen as the correlation between one normalized vector and another normal vector sampled uniformly from the orthogonal complement. From this point view, the moment is independent of the matrix structure, and the results of Lemma 3 can be extended to the case where $(W_{ij})_{n \times m}$ satisfies 
$$WW^T = I,\ \ \ \ n \leq m.$$
Thus, we believe that the previous theorems also hold without assuming that the weight matrix is a  square matrix.
\end{rem}

\subsection{NTK during Training}

The NTK of networks with orthogonal initialization stays constant during training in the infinite-width limit. More precisely, we can give an upper bound for the change of parameters and NTK at wide width.
  \begin{thm}\label{thm:main}
            Assume that  $\lambda_{\rm min}(\Theta_\infty) >0$ and $\eta_{\rm critical} = \frac{\lambda_{\rm min}(\Theta_\infty)+\lambda_{\rm max}(\Theta_\infty)}{2}$. Let $n = n_1, ..., n_{L-1}$ be the width of hidden layers. Consider a FCN at orthogonal initialization, trained by gradient descent with learning rate $\eta < \eta_{\rm critical}$ (or gradient flow). For every input $x\in \mathbb R^{n_0}$ with $\|x\|_2\leq 1$, with probability arbitrarily close to 1,  
           \begin{equation}
            \sup_{t\geq 0}\frac{\left\|\theta_t -\theta_0\right\|_2}{\sqrt n},
            \,\, 
            \sup_{t\geq 0}\left\|\hat \Theta_t - \hat \Theta_0 \right\|_F = O(n^{-\frac 1 2}), \,\, {\rm as }\quad n\to \infty\,. 
           \end{equation}
 where $\hat \Theta_t$  are empirical kernels of network with finite width.
 
 For a CNN at orthogonal initialization, trained by gradient descent with learning rate $\eta < \eta_{\rm critical}$ (or gradient flow), for every input $x\in \mathbb R^{n_0}$ with $\|x\|_2\leq 1$, and filter relative spatial location $\beta \in  \{-k,\dots, 0, \dots, k \}$, with probability arbitrarily close to 1,
           \begin{equation}
            \sup_{t\geq 0}\frac{\left\|\theta_{\beta,t} -\theta_{\beta,0} \right\|_2}{\sqrt n},
            \,\, 
            \sup_{t\geq 0}\left\| \hat \Theta_t -   \hat \Theta_0 \right\|_F = O(n^{-\frac 1 2}).
           \end{equation}
 \end{thm} 

We fist prove the {\it local Lipschitzness} of $J(\theta) :=  \nabla_{\theta} h^L(\theta,X) \in \mathbb{R}^{(Dn_L) \times |\theta|} $, where $|\theta|$ is the number of parameters. Once we have the upper bound of $J(\theta)$, the stability comes from proving that the difference is small of the NTK form time $t$ to time $t+1$. Our primary target is construct local Lipschitzness of orthogonal initialization. The Lemma below illustrates the local Lipschitzness of $J(\theta)$ without scale condition for the input layer.
\begin{lem}[local Lipschitzness of $J(\theta)$]
$\exists K > 0$ s.t for every $c > 0$, $\exists n_c$ s.t for $n \geq n_c$, we have the following bound:
\begin{equation}
\left\{
             \begin{array}{lr}
             ||J(\theta) - J(\tilde{\theta})||_F \leq K ||\theta - \tilde{\theta}||_2 , &  \\
               ~~~~~~~~~~~~~~~~~~~~~~~~~~~~~~~~~~~~~~~~~~~~~~~~~~~~~~~~~~~~~~~,  \forall \theta, \tilde{\theta} \in B(\theta_0 , \frac{c}{n})& \\
            ||J(\theta)||_F \leq K, 
             \end{array}
\right.
\end{equation}
Where $$B(\theta_0 , R) := \{\theta: ||\theta - \theta_0 ||_2 < R \}.$$
\end{lem}
\begin{proof} We prove the result by induction for the standard parameterization, while the proof for ntk parameterization can be derived in the same way.
For $l \geq 1$, let
$$\delta^l (\theta , x) :=  \nabla_{h^l (\theta , x)} h^L (x) \in \mathbb{R}^{n_L n} $$
$$\delta^l (\theta , X) :=  \nabla_{h^l (\theta , X)} h^L (X) \in \mathbb{R}^{(n_L D) \times(n_L D)} $$
Let $\theta = \{W^l , b^l \}$,  $\tilde{\theta} = \{\tilde{W}^l , \tilde{b}^l\}$ be two points in $B( \theta_0 , \frac{c}{n})$. Since the initialization is to choose an orthogonal matrix, i.e. $W^T W = \sigma_w^2{\bf I}$. If the width $n$ is large enough, we have,
$$|| W^l ||_{op} , ~~|| \tilde W^l ||_{op} \leq 3\sigma_w  {\rm~~~~~ for~~ all~~} l.\\$$
As in the original proof for Gaussian initialization \cite{lee2019wide}, there is a constant $K_1$, depending on $D$, $\sigma^2_w$, $\sigma_b^2$, and number of layers $L$ s.t with high probability over orthogonal initialization,
$$||x^l ( \theta , X ) ||_2 , \ \ \ \ \ ||\delta^l (\theta , X) ||_2 \leq K_1$$
$$||x^l ( \theta , X ) - x^l ( \tilde{\theta} , X )||_2 , \ \ \ ||\delta^l  ( \theta , X ) - \delta^l  ( \tilde{\theta} , X )||_2 \leq K_1 ||\tilde{\theta} - \theta ||_2$$
Note: there is a scaling factor $\frac{1}{\sqrt{n}}$ along with $||x^l (\theta , X)||_2$  in the Gaussian case, since from the input layer to the first layer, for the Gaussian initialization, we have
$$|| W^1 ||_{op} , \ \ \ \ \ \ \ \ \ \  || \tilde W^1 ||_{op} \leq 3 \sigma_w \frac{\sqrt{n}}{\sqrt{n_0}}.$$

Decomposing the $J(\theta)$ into two parts, we have
\begin{align*}
||J(\theta)||^2_F &= \displaystyle\sum_{l}|| \nabla_{W^l} h^{L} (\theta)||_F^2 + || \nabla_{b^l} L^{h} (\theta)||_F^2\\
&=  \displaystyle\sum_{l} \displaystyle\sum_{x \in X} ||x^{l-1} (\theta , x) \delta^l (\theta , x)^T||^2_F + ||\delta^l (\theta , x)^T||^2_F\\
& \leq \displaystyle\sum_{l} \displaystyle\sum_{x \in X} (1 + ||x^{l - 1} (\theta , x) \delta^l (\theta , x)^T||^2_F )\cdot  ||\delta^l (\theta , x)^T||^2_F \\
& \leq \displaystyle\sum_{l} (1 + K_1^2 ) \displaystyle\sum_{x} ||\delta^l (\theta , x)^T||^2_F \\
& \leq 2L  K_1^4
\end{align*}
and similarly,
\begin{align*}
||J(\theta) - J(\tilde{\theta})||^2_F & = \displaystyle\sum_{l} \displaystyle\sum_{x \in X} ||x^{l - 1} (\theta , x) \delta^l (\theta , x)^T - x^{l - 1} (\tilde{\theta} , x) \delta^l (\tilde{\theta} , x)^T||^2_F + ||\delta^l (\theta , x)^T - \delta^l (\tilde{\theta} , x)^T ||^2_F \\
& = \displaystyle\sum_{l} \displaystyle\sum_{x \in X} ||x^{l - 1} (\theta , x) \delta^l (\theta , x)^T - x^{l - 1} (\tilde{\theta} , x) \delta^l (\theta , x)^T ||^2 \\
& + ||x^{l - 1} (\tilde{\theta} , x) \delta^l (\theta , x)^T -x^{l - 1} (\tilde{\theta} , x) \delta^l (\tilde{\theta} , x)^T||^2 + ||\delta^l (\theta , x)^T - \delta^l (\tilde{\theta} , x)^T ||^2 \\
& \leq  \displaystyle\sum_{l} 2 K_1^4 ||\tilde{\theta} - \theta ||^2  + K_1^2 ||\tilde{\theta} - \theta ||^2 \\
&\leq 3 K_1^4 L  ||\tilde{\theta} - \theta ||^2 
\end{align*}
\end{proof}
Note: In our setting, since we want the orthogonal and Gaussian initialization have the same NTK, we need to scale the initialization for the input layer:
$$(W^1)^T W^1 =\sigma^2_w \frac{n}{n_0},\ \ \ \ \ W^1 \in \mathbb{R}^{n \times n_0}$$
If we impose this condition, then,
$$||W^1 ||_{op}, ~~~~|| \tilde W^1 ||_{op} \leq 3\sigma_w \frac{\sqrt{n}}{\sqrt{n_0}} $$
So we recover the same lipschitz constant as in the Gaussian initialization condition:
\begin{cor}(local Lipschitzness of $J(\theta)$ with scale condition)
Under the above scaling condition on the input layer, $\exists K > 0$ s.t for every $c > 0$, $\exists n_c$ s.t for $n \geq n_c$, we have the following bound:
$\exists K > 0$ s.t for every $c > 0$, $\exists n_c$ s.t for $n \geq n_c$, we have the following bound:
\begin{equation}
\left\{
             \begin{array}{lr}
               \frac{1}{\sqrt{n}}||J(\theta) - J(\tilde{\theta})||_F \leq K ||\theta - \tilde{\theta}||_2 , &  \\
               ~~~~~~~~~~~~~~~~~~~~~~~~~~~~~~~~~~~~~~~~~~~~~~~~~~~~~~~~~~~~~~~,  \forall \theta, \tilde{\theta} \in B(\theta_0 , \frac{c}{\sqrt{n}})& \\
             \frac{1}{\sqrt{n}}||J(\theta)||_F \leq K, 
             \end{array}
\right.
\end{equation}
Where $B(\theta_0 , R) := \{\theta: ||\theta - \theta_0 ||_2 < R \}$.
\end{cor}
With this Corollary, the left proof for Theorem \ref{thm:main} is exact same as that of Theorem G.1 and Theorem G.2 for Gaussian initialization in \cite{lee2019wide}.\\

For the CNN case, we refer the reader to \cite{liu2020linearity}, especially $P_{26} \sim P_{29}$. The constancy of NTK can be reduced to control of the Hessian norm. Most of the estimates are independent of the initialization, we only need to replace Lemma F.4 of \cite{liu2020linearity} into the following lemma:
 \begin{lem}\label{lemma:everyactivation}
	For any $l \in [L]$, given $i\in [m]$, with probability at least $1 - 2e^{-{c_\alpha^{(l)}\ln^2(n_{l-1})}} - \frac{2\sqrt{3}}{n_{l-1} -1}$ for some constant $c_\alpha^{(l)}>0$, the vector norm of the pre-activation $|h_{i}^{(l)}| = \tilde{O}(1)$ at initialization.
\end{lem}
\begin{proof}
When $l =1$,
\begin{align*}
|h_i^1| &= |\phi(\frac{1}{n_0}\sum_{k=1}^{n_0}W_{ik}^1 x_k)| \\
& \leq |\frac{m}{\sqrt{n_0}}\sum_{k=1}^{n_0}W_{ik}^1 x_k| + |\phi(0)|.
\end{align*}
From \cite{chatterjee2007multivariate} or \cite{du}, by Stein's method, we have a total variation bound:
$$d(\sum_{k=1}^{n_0}W_{ik}^1 x_k, Z)_{TV} \leq \frac{2\sqrt{3}}{n_{0} -1} \ ,$$
where $Z \sim N(0, \norm{x}^2)$. By the definition of total variation distance, we conclude that under $1 - \frac{2\sqrt{3}}{n_{0} -1}$ probability,  $\sum_{k=1}^{n_0}W_{ik}^1 x_k \sim N(0, \norm{x}^2)$. Hence by the concentration inequality for Gaussian random variable, we have
\begin{align*}
\mathbb{P}[|h_i^1| &\ge \ln(d_0) + |\phi(0)|]\\
& \leq \mathbb{P}[|\frac{m}{\sqrt{n_0}}\sum_{k=1}^{n_0}W_{ik}^1 x_k| \ge \ln(d_0)] \\
& \leq 2 e^{\frac{-d_0 \ln^2 (d_0)}{2 m^2 \norm{x}^2}}.
\end{align*}
Since $\frac{d_0}{m^2 \norm{x}^2}$ is of order $O(1)$, we know that
$$|h_i^1| \sim O(\ln (d_0)).$$
When $l \ge 2$, by induction, we can prove the lemma along the same line as $l = 1$.
\end{proof}
Combing the lemma with the argument in \cite{liu2020linearity}, we have proved theorem 3 in the main text.\\
We remark that to further extend other results of \cite{liu2020linearity}, we also need a Chi-Squared Stein method for nonlinear functional of the orthogonal matrix ensemble, which can be achieved by the exchangeable pair method in \cite{du}. 
\subsection{Additional Numerical Experiments}

\begin{figure*}[t!]
\centering
  \centering
  \includegraphics[width=1.0\textwidth]{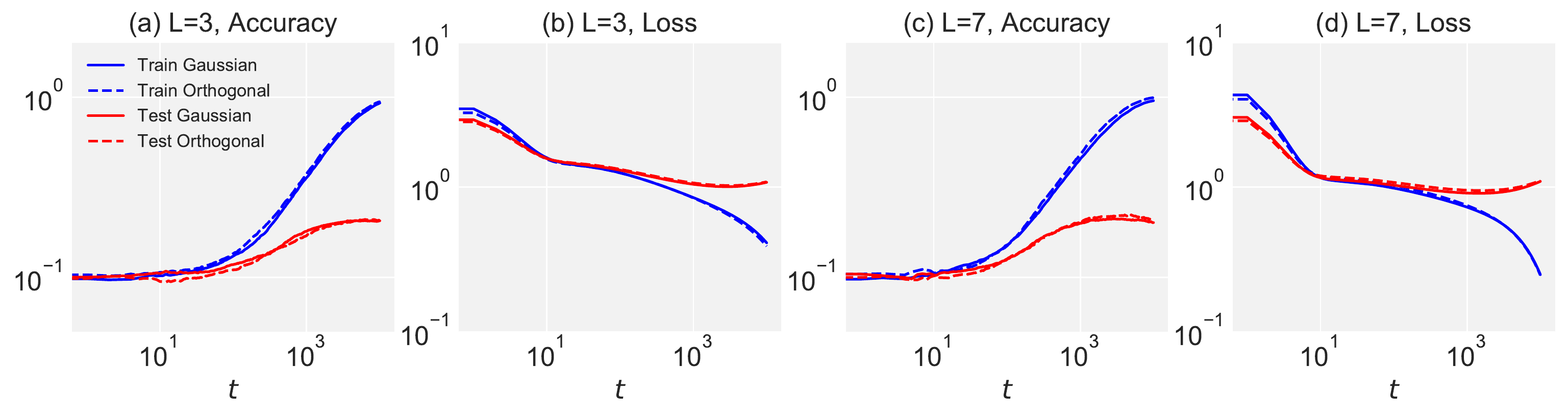}
 \caption{Dynamics of full batch gradient descent on Gaussian and orthogonally initialized networks of $T=10^4$ steps. Orthogonal networks behaves similarly to dynamics on the corresponding Gaussian networks, for loss and accuracy functions. The dataset is selected from full CIFAR10 with $D=256$, while MSE loss and tanh fully-connected networks are adopted for the classification task. (a,b) Network with depth $L=3$ and width of $n = 400$, with $\sigma^2_w = 1.5$, and $\sigma^2_b =0.01$. (c,d) Network with depth $L=7$ and width of $n = 800$, with $\sigma^2_w = 1.5$, and $\sigma^2_b =0.1$. While the solid lines stand for Gaussian weights, dotted lines represent orthogonal initialization.}
 \label{fig:1s}
\end{figure*}

\begin{figure*}[t!]
\centering
  \centering
  \includegraphics[width=1.0\textwidth]{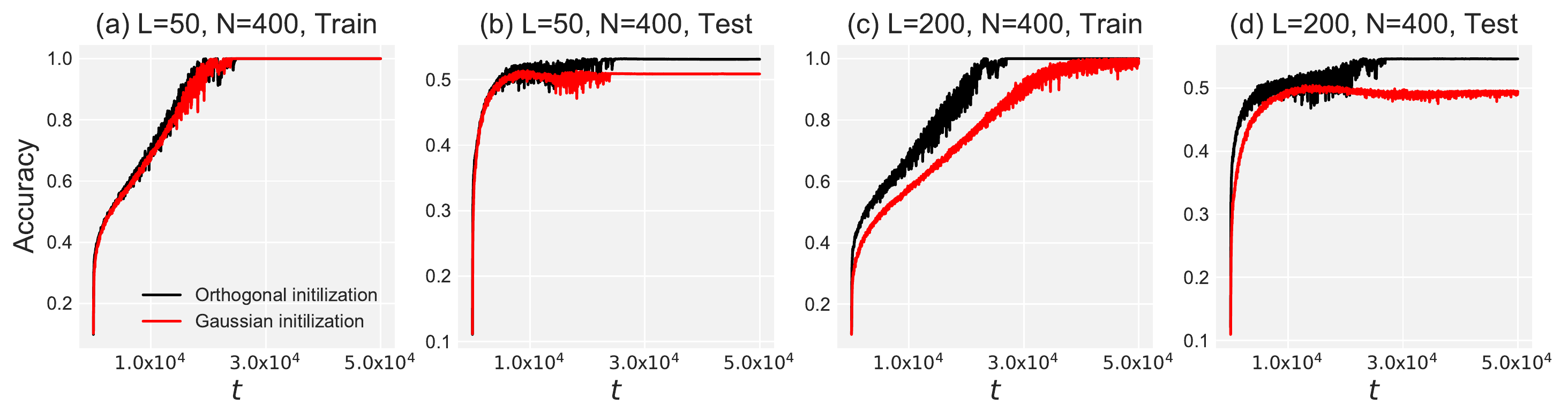}
 \caption{Learning dynamics measured by the optimization and generalization accuracy on train set and test set, for networks of (a,b) depth $L=50$, width $n=400$, and learning rate $\eta=0.01$ (c,d) depth $L=200$, width $n=400$, and learning rate $\eta = 0.004$. We additionally average our results over 30 different instantiations of the network to reduce noise. Black curves are the results of orthogonal initialization and red curves are performances of Gaussian initialization. (a,c) Training speed of an orthogonally initialized network is faster than that of a Gaussian initialized network. (b,d) On the test set, compared to the network with Gaussian initialization, the orthogonally initialized network not only learns faster, but ultimately converges to a higher generalization performance.}
 \label{fig:2s}
\end{figure*}

In the first experiment, summarized in Figure \ref{fig:1s}, we compare the train and test loss and accuracy with two different initialization, i.e., Gaussian and orthogonal weights using $D=256$ samples from the CIFAR10 dataset. To reduce noise, we averaged the results over 30 different instantiations of the networks. Figures \ref{fig:1s}(a,b) show the results of the experiments on the $L=3$, $n=400$ network with tanh activation, while figures \ref{fig:1s}(c,d) display the results for the $L=7$, $n=800$ network with tanh activation. All networks are optimized using gradient descent with a learning rate of $\eta=10^{-4}$ for $T=10^4$ steps. Consistent with our theoretical findings, the loss and accuracy of both networks are almost the same {\bf in the NTK regime}.

As for learning dynamics {\bf outside the NTK regime}, we preform more experiments on the convergence properties of networks in the large depth and large learning rate phase. Figure \ref{fig:2s}  not only confirms that the orthogonally initialized networks train faster and learn better than the orthogonally initialized network, but the performance gap between the two initializations can be very large as the depth increases.

\bibliographystyle{unsrt}
\bibliography{ref.bib}
\end{document}